\documentclass[journal,twocolumn]{IEEEtran}

\usepackage{amsmath,amssymb,amsfonts,amsthm}
\usepackage{graphicx}
\usepackage{booktabs}
\usepackage{cite}
\usepackage{xcolor}
\usepackage[hidelinks]{hyperref}
\usepackage{algorithm}
\usepackage{algpseudocode}
\usepackage{multirow}
\usepackage{array}
\usepackage{url}
\usepackage{longtable}

\newtheorem{theorem}{Theorem}

\newtheorem{proposition}{Proposition}

\theoremstyle{definition}
\newtheorem{definition}{Definition}
\newtheorem{remark}{Remark}

\newcommand{\R}{\mathbb{R}}
\newcommand{\E}{\mathbb{E}}
\newcommand{\Pbb}{\mathbb{P}}

\DeclareMathOperator{\Cov}{Cov}

\DeclareMathOperator{\Span}{span}

\DeclareMathOperator*{\esssup}{ess\,sup}

\newcommand{\ind}{\mathbf{1}}

\newcommand{\cF}{\mathcal{F}}
\newcommand{\cY}{\mathcal{Y}}

\newcommand{\cS}{\mathcal{S}}
\newcommand{\cM}{\mathcal{M}}
\usepackage{tikz}
\usetikzlibrary{shapes.arrows, positioning, shadows.blur, fit, backgrounds, matrix, calc}

\title{The Risk Shadow of Principal Component Analysis: When 99.9999\% Variance Preservation Causes Catastrophic Decision Errors }

\author{
Hamidou~Tembine%
\thanks{Department of EECS, School of Engineering, UQTR, Canada.}
\thanks{Learning and Game Theory Laboratory (LnG Lab), TIMADIE.}
\thanks{Contact: tembine@ieee.org}
}

\begin{document}

\maketitle

\begin{abstract}
Principal Component Analysis (PCA) preserves variance, not the information needed to detect rare catastrophic events. This paper proves the existence of a {\it Risk Shadow}: PCA can retain over 99.9999 percent of total variance while completely erasing all signal about rare, high-impact failures. When this happens, even the best possible classifier operating on the PCA representation reduces to a constant predictor. The root cause is a fundamental mismatch between variance maximization and tail risk awareness. To break the shadow, we introduce Expectile PCA (ExPCA) and Tail-Preserving PCA (TP-PCA), two methods that reweight the data covariance toward high-impact events. We prove theoretically that ExPCA strictly outperforms PCA in retaining rare-event information, and we validate our claims on synthetic data and a real-world credit card fraud detection benchmark. Our results call for a fundamental rethinking of variance-based dimensionality reduction in high-stakes decisions.
\end{abstract}


\section{Introduction}

Few ideas have influenced modern data assimilation, statistics, machine learning, signal processing, econometrics, computer vision, quantitative finance, genomics, and machine intelligence as deeply as Principal Component Analysis (PCA). Since the pioneering works of \cite{old1,old2}, PCA has become the canonical method for dimensionality reduction because it solves a fundamental geometric problem: among all linear subspaces of a prescribed dimension, it finds the one that minimizes mean-square reconstruction error. It identifies the directions of maximal variance and produces the optimal low-dimensional representation in the $L^2$ sense. The success of PCA is difficult to overstate. It is routinely deployed as a {\bf preprocessing} layer before classification, clustering, anomaly detection, forecasting, compression, reinforcement learning, and deep neural architectures. In countless applications, reducing dimensionality through PCA improves computational efficiency, suppresses noise, mitigates overfitting, and facilitates visualization. Its  computational scalability, and spectral optimality have made PCA one of the most widely used algorithms in the history of data engineering and data science. Yet PCA was never designed to solve a decision problem.
This distinction is often overlooked. PCA optimizes a geometric objective defined entirely through second-order moments of the data distribution. Classification systems, however, are evaluated through decision losses. Medical diagnosis is judged by missed cancers, not reconstruction error. Fraud detection is judged by undetected fraud, not variance explained. Autonomous driving is judged by collisions avoided, not spectral fidelity. In these settings, the quantity that matters operationally is not variance preservation but decision risk.
The central question therefore becomes:

\begin{quote}
Can a representation that is geometrically optimal be catastrophically suboptimal for decision making?
\end{quote}

This paper shows that the answer is unequivocally yes.
The reason is structural. Variance and decision risk are fundamentally different  objects. Variance is an unconditional second-moment quantity. Misclassification cost is a conditional decision-theoretic quantity that depends jointly on the data distribution, the label distribution, the operational loss function, and the user's tolerance to errors. Nothing in the PCA objective incorporates labels, costs, asymmetries, rare events, or downstream consequences. There is no theoretical mechanism forcing directions of maximal variance to coincide with directions that minimize decision risk. This observation immediately raises a deeper question. If PCA ignores the user's loss function, can it nevertheless preserve the information required for optimal decisions? The answer is generally negative.
Indeed, the dominant variance directions of a distribution need not contain any information about the target variable. The information necessary for classification may reside entirely in directions carrying negligible variance. When such a configuration occurs, PCA faithfully preserves the geometry of the data while simultaneously destroying the information required for the decision task. The resulting representation may appear statistically excellent according to classical variance-retention criteria while being operationally useless.
This phenomenon is particularly severe in rare-event learning.
Across modern high-stakes systems, the most consequential events are often extremely infrequent. Financial crashes \cite{li2020} occupy tiny portions of market trajectories. Fraudulent transactions constitute a minute fraction of all payments. Early-stage diseases may appear as weak perturbations hidden beneath dominant physiological variability\cite{metabolites2025}. Safety-critical failures in autonomous systems arise from rare combinations of environmental conditions \cite{ieeetifs2025}. In all these situations, the signals carrying the highest operational value frequently contribute only a negligible fraction of the total variance. As a consequence, maximizing explained variance can become fundamentally misaligned with minimizing catastrophic errors.  The directions most relevant for decision-making are often those least relevant for variance maximization.

The implications are profound. A representation may preserve virtually all observable variance and yet completely eliminate the information needed to detect the events that matter most. We provide an example where PCA retains more than $99.9999\%$ of the total variance while erasing all information regarding a critical rare class. After projection, the Bayes-optimal classifier collapses to a constant predictor despite the apparent preservation of nearly the entire dataset geometry. The representation remains spectrally optimal but becomes decision-theoretically worthless.  This reveals a previously unrecognized failure mode of variance-based dimensionality reduction.
We refer to this phenomenon as the \emph{Risk Shadow}. The Risk Shadow occurs when a representation preserves nearly all global variance while eliminating the information required to control rare-event decision risk. Under the Risk Shadow, nominal performance indicators can remain deceptively strong. Overall accuracy may approach unity because the dominant class is correctly predicted almost everywhere. Yet tail risk, expected catastrophic loss, and high-cost error exposure can increase dramatically. In our analytical benchmarks, a representation preserving more than $99.9999\%$ of the total variance produces an increase of approximately $890\%$ in expectile-based misclassification risk relative to a risk-aware alternative. The significance of this result extends beyond PCA itself\cite{old3,old4}. Since its introduction more than a century ago, PCA has evolved into one of the most influential paradigms in statistical learning, inspiring a vast ecosystem of extensions designed to address nonlinearity, robustness, sparsity, probabilistic uncertainty, structured data, functional observations, streaming environments, manifold constraints, and high-dimensional inference. Over the last decades, numerous variants have been proposed, including Kernel PCA, Sparse PCA, Robust PCA, Probabilistic PCA, Functional PCA, Tensor PCA, and many other extensions. Despite their methodological differences, these approaches remain rooted in a common paradigm: the representation is selected through a geometric criterion, after which a downstream decision system is trained on the compressed coordinates. The underlying optimization objective remains disconnected from the operational loss ultimately used to evaluate performance.
A first major line of research sought to overcome the intrinsic linearity of PCA through nonlinear embeddings. Kernel PCA generalized principal component analysis to reproducing kernel Hilbert spaces, enabling nonlinear feature extraction while preserving the spectral foundations of the original method \cite{scholkopf1998}. Subsequent developments included manifold learning and spectral approaches such as Isomap \cite{tenenbaum2000}, Locally Linear Embedding (LLE) \cite{roweis2000}, Laplacian Eigenmaps \cite{belkin2003}, Hessian Eigenmaps \cite{donoho2003}, Diffusion Maps \cite{coifman2006}, and various nonlinear principal manifold methods.
A second research direction focused on interpretability and high-dimensional statistics. Sparse PCA introduced sparsity constraints on loading vectors to improve interpretability and variable selection \cite{zou2006,daspremont2007,journee2010}. Related formulations include structured sparse PCA, group sparse PCA, elastic-net PCA, nonnegative PCA, and semidefinite relaxations for sparse component estimation.
A third family addressed robustness against noise, contamination, and outliers. Robust PCA methods replace classical covariance estimation by robust alternatives or explicitly decompose data into low-rank and sparse components \cite{candes2011,xu2010,hubert2005}. Numerous variants subsequently emerged, including M-estimator PCA, projection-pursuit PCA, L1-PCA, outlier-resistant PCA, and low-rank-plus-sparse decomposition frameworks.
Probabilistic formulations constitute another major branch. Probabilistic PCA (PPCA) introduced a latent-variable generative model interpretation of PCA \cite{tipping1999}, later extended through Bayesian PCA \cite{bishop1999}, Variational PCA, Factor Analysis, Bayesian nonparametric latent-factor models, and probabilistic matrix factorization techniques.
 Functional data analysis generated Functional PCA (FPCA), where observations are viewed as random functions rather than finite-dimensional vectors \cite{ramsay2005}. This framework has become central in longitudinal analysis, biomedical signals, climate modeling, and spatiotemporal systems. Extensions include Multilevel FPCA, Dynamic FPCA, Sparse FPCA, and Kernel FPCA.
As multidimensional datasets became prevalent, tensor-based generalizations emerged. Tensor PCA, multilinear PCA, higher-order singular value decomposition (HOSVD), Tucker decompositions, CANDECOMP/PARAFAC models, multilinear subspace learning, and low-rank tensor recovery methods were developed to preserve multiway structure without vectorization \cite{lu2008,kolda2009}.
 The growth of large-scale data motivated incremental and online variants. Incremental PCA \cite{hall1998}, online PCA \cite{oja1982}, stochastic PCA, streaming PCA, distributed PCA, randomized PCA \cite{halko2011}, sketching-based PCA, and communication-efficient PCA were developed to cope with modern computational constraints.
Further extensions include Independent Component Analysis (ICA) inspired spectral methods \cite{hyvarinen2001}, Principal Curves and Principal Surfaces \cite{hastie1989}, Local PCA, Mixture PCA, Hierarchical PCA, Multiscale PCA, Graph PCA, Geodesic PCA, Riemannian PCA, Grassmannian PCA, Complex PCA, Quaternion PCA, Dynamic PCA, Evolutionary PCA, Contrastive PCA \cite{abid2018}, Fair PCA \cite{samadi2018}, Deep PCA, Autoencoder-based PCA, Supervised PCA \cite{bair2006}, Discriminative PCA, Canonical PCA, Generalized PCA, Weighted PCA, Missing-Data PCA, Compositional PCA, and numerous domain-specific adaptations.

The present work challenges this paradigm.
Rather than asking which subspace best reconstructs the data, we ask which subspace best preserves the information required to minimize high-impact decision errors. This shift transforms dimensionality reduction from a geometric compression problem into a decision-theoretic optimization problem. To investigate this question, we establish a framework connecting dimensionality reduction, information preservation, tail risk, and misclassification costs. We prove that variance maximization and risk minimization can be structurally incompatible. We derive exact conditions under which PCA completely erases rare-event information despite preserving essentially all variance. We then introduce risk-aware alternatives based on expectile principles that directly incorporate tail-sensitive objectives into the representation-learning stage.

Our analysis reveals a fundamental principle:
\begin{quote}
{\it A representation should not be judged solely by how much variance it preserves, but by how much decision-relevant information it retains.}
\end{quote}
For high-stakes machine intelligence systems, variance is not  the only objective. Decisions and consequences of decisions are.
Diagram \ref{refdiag} maps a structural comparison between standard PCA and the tail-risk-aware exp2PCA. It illustrates how standard variance maximization inadvertently erases critical rare-event signals to cause catastrophic tail-risk, and how the proposed alternative preserves those signals to secure robust classification.
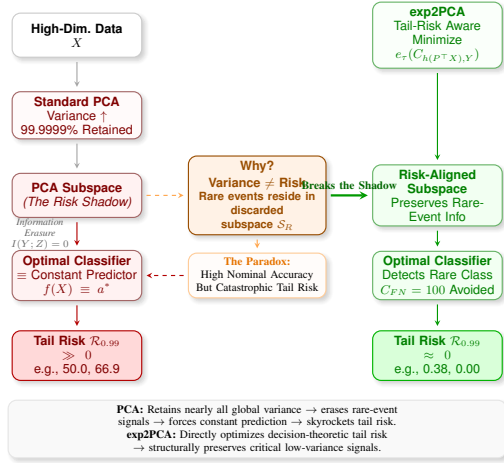
\begin{figure}[htb]
\centering
\resizebox{0.4\textwidth}{!}{%
\begin{tikzpicture}[
    font=\sffamily\small,
    >=stealth,
    base/.style={rectangle, draw=gray!30, thick, rounded corners=6pt, align=center, minimum height=1.1cm, text width=3cm, blur shadow={shadow blur steps=5, shadow blur radius=3pt, shadow opacity=20}},
    neutral/.style={base, fill=white, draw=gray!40, text=black},
    risk/.style={base, fill=red!5, draw=red!40!black!60, text=red!50!black},
    good/.style={base, fill=green!5, draw=green!50!black!60, text=green!50!black},
    accent/.style={base, fill=orange!10, draw=orange!70!black!50, text=orange!60!black, font=\sffamily\small\bfseries},
    arrow/.style={thick, gray!70, ->, shorten >=3pt, shorten <=3pt},
    risk_arrow/.style={thick, red!70!black, ->, shorten >=3pt, shorten <=3pt},
    good_arrow/.style={thick, green!60!black, ->, shorten >=3pt, shorten <=3pt},
    diag_label/.style={font=\scriptsize\itshape, text=gray!80!black, align=center}
]

\colorlet{amber}{orange!20!yellow}

\node[font=\Large\bfseries\sffamily, text=gray!90!black, anchor=center] (title) at (0, 5.5) {PCA’s Responsibility in Increasing Misclassification Risk};

\node[neutral] (data) at (-4.5, 3.5) {\textbf{High-Dim. Data} \\ $X$};

\node[risk, fill=red!2] (pca) at (-4.5, 1.5) {\textbf{Standard PCA} \\ Variance $\uparrow$ \\ 99.9999\% Retained};

\node[risk] (pca_proj) at (-4.5, -0.5) {\textbf{PCA Subspace} \\ \textit{(The Risk Shadow)}};

\node[risk] (const) at (-4.5, -2.5) {\textbf{Optimal Classifier} \\ $\equiv$ Constant Predictor \\ $f(X) \equiv a^*$};

\node[risk, fill=red!15, draw=red!70!black, thick] (tail_pca) at (-4.5, -4.5) {\textbf{Tail Risk $\mathcal{R}_{0.99}$} \\ $\gg 0$ \\ e.g., 50.0, 66.9};

\draw[arrow] (data) -- (pca);
\draw[arrow] (pca) -- (pca_proj);
\draw[risk_arrow] (pca_proj) -- node[left, diag_label, xshift=-2pt] {Information \\ Erasure \\ $I(Y;Z)=0$} (const);
\draw[risk_arrow] (const) -- (tail_pca);

\node[accent, fill=orange!10, draw=orange!60!black, text width=3.2cm] (why) at (0, -0.5) {
    \textbf{Why?} \\ 
    \vspace{2pt}
    \small Variance $\neq$ Risk \\
    \footnotesize Rare events reside in \\ discarded subspace $\mathcal{S}_R$
};

\node[neutral, draw=orange!40, fill=orange!2, text width=3.2cm, font=\footnotesize] (gap) at (0, -2.5) {
    \textcolor{orange!90!black}{\textbf{The Paradox:}} \\ High Nominal Accuracy \\ But Catastrophic Tail Risk
};

\draw[arrow, dashed, draw=orange!70] (pca_proj) -- (why);
\draw[arrow, draw=orange!70] (why) -- (gap);

\node[good, yshift=2cm, fill=green!2] (exp2) at (4.5, 1.5) {\textbf{exp2PCA} \\ Tail-Risk Aware \\ Minimize $e_\tau(C_{h(P^\top X),Y})$};

\node[good] (exp2_proj) at (4.5, -0.5) {\textbf{Risk-Aligned Subspace} \\ Preserves Rare-Event Info};

\node[good] (class) at (4.5, -2.5) {\textbf{Optimal Classifier} \\ Detects Rare Class \\ $C_{FN}=100$ Avoided};

\node[good, fill=green!15, draw=green!70!black, thick] (tail_exp) at (4.5, -4.5) {\textbf{Tail Risk $\mathcal{R}_{0.99}$} \\ $\approx 0$ \\ e.g., 0.38, 0.00};

\draw[good_arrow] (exp2) -- (exp2_proj);
\draw[good_arrow] (exp2_proj) -- (class);
\draw[good_arrow] (class) -- (tail_exp);

\draw[good_arrow, ultra thick] (why.east) -- node[above, text=green!40!black, font=\footnotesize\bfseries, sloped] {Breaks the Shadow} (exp2_proj.west);
\draw[risk_arrow, dashed] (gap.west) -- (const.east);

\begin{scope}[on background layer]
    \node[fill=red!8!white, draw=red!20, dashed, rounded corners=10pt, fit=(pca) (tail_pca), label={[text=red!70!black, font=\bfseries\small]above:PCA FAILURE DOMAIN}] (left_group) {};
    
    \node[fill=green!8!white, draw=green!20, dashed, rounded corners=10pt, fit=(exp2) (tail_exp), label={[text=green!70!black, font=\bfseries\small]above:exp2PCA SUCCESS PATH}] (right_group) {};
\end{scope}

\node[neutral, fill=gray!5, draw=gray!20, text width=12.2cm, font=\footnotesize, anchor=north] (footer) at (0, -5.6) {
    \textbf{PCA:} Retains nearly all global variance $\rightarrow$ erases rare-event signals $\rightarrow$ forces constant prediction $\rightarrow$ skyrockets tail risk. \\
    \textbf{exp2PCA:} Directly optimizes decision-theoretic tail risk $\rightarrow$ structurally preserves critical low-variance signals.
};

\end{tikzpicture}
} \caption{PCA can retain 99.9999 \% of the global variance and be responsible of 890\% increase in expectile risk of misclassification cost.} \label{refdiag}
\end{figure}

\subsubsection*{Related Work on Risk-Aware and Decision-Aware Dimensionality Reduction}
The relationship between dimensionality reduction and downstream decision risk has attracted increasing attention over the last three decades. Nevertheless, most existing approaches remain rooted in geometric, statistical, or probabilistic objectives rather than directly optimizing operational decision losses.

A first line of research addresses robustness to contamination and adversarial observations. Robust PCA (RPCA) seeks to recover low-rank structure in the presence of sparse corruptions, gross outliers, or missing observations \cite{candes2011,xu2012,hubert2005}. These methods significantly improve resilience against data contamination and have become foundational in computer vision, anomaly detection, and signal processing. However, their notion of robustness is fundamentally statistical rather than operational: the objective remains the recovery of an underlying low-rank geometric representation. RPCA does not distinguish between errors that are operationally benign and errors that may trigger catastrophic decision consequences.

A second direction focuses on interpretability and high-dimensional variable selection. Sparse PCA and its numerous extensions introduce sparsity-inducing penalties to obtain principal components involving only a limited subset of variables \cite{zou2006,daspremont2007,journee2010}. Although these methods improve explainability and facilitate scientific interpretation, the optimization criterion remains variance maximization under structural constraints. The downstream cost associated with misclassification, false alarms, or rare-event failures does not enter the representation-learning objective.

A third family incorporates supervision into dimensionality reduction. Sufficient Dimension Reduction (SDR), including Sliced Inverse Regression (SIR) \cite{li1991}, Sliced Average Variance Estimation (SAVE) \cite{cook1991}, Principal Hessian Directions \cite{li1992}, and related inverse-regression techniques, aims to preserve the conditional distribution of the response variable. More recent developments include supervised PCA \cite{bair2006}, discriminative PCA, partial least squares, and task-guided latent-variable models. While these methods leverage label information, they typically target conditional sufficiency, correlation structure, or prediction accuracy. They do not explicitly account for asymmetric loss functions, rare-event costs, tail-risk exposure, or user-specific operational penalties. Moreover, many classical SDR approaches rely on restrictive assumptions such as elliptically distributed predictors, linearity conditions, or covariance invertibility.

A fourth line of research studies rare events and extreme outcomes directly. Extreme Value Theory (EVT) provides a framework for modeling tail behavior, rare exceedances, and catastrophic events \cite{embrechts1997,dehaan2006,cloteaux2023}. EVT-based methods have found applications in finance, cybersecurity, climate science, and reliability engineering. Similarly, importance-weighting techniques, cost-sensitive learning, focal losses, and imbalance-aware classification strategies attempt to compensate for rare-event underrepresentation \cite{liu2020,he2009,lin2017}. However, these approaches generally operate after the representation has already been constructed. The feature space itself is typically inherited from variance-driven, likelihood-driven, or task-agnostic dimensionality reduction procedures. The representation may already have discarded the very directions required to identify rare but consequential events.

A fifth body of work investigates coherent and tail-sensitive risk measures. Value-at-Risk (VaR), Conditional Value-at-Risk (CVaR), spectral risk measures, distortion risk measures, and expectiles have become central tools in modern risk management \cite{artzner1999,acerbi2002,newey1987,bellini2014}. In particular, expectiles possess attractive theoretical properties including coherence under suitable conditions, elicitability, differentiability, and strong sensitivity to tail outcomes \cite{newey1987,bellini2014}. Expectile-based methods have been widely studied in econometrics, finance, actuarial science, and risk forecasting. Nevertheless, their use has been largely restricted to estimation, regression, forecasting, and portfolio optimization. To the best of our knowledge, expectiles have not been systematically employed as the primary objective for spectral representation learning in the context of dimensionality reduction.

Several recent developments have also attempted to align representations with downstream tasks through metric learning, information bottleneck principles, contrastive objectives, mutual-information maximization, and deep representation learning \cite{tishby2000,hadsell2006,oord2018}. While these methods introduce task awareness, their objectives remain fundamentally different from direct minimization of tail-sensitive decision costs. Information preservation, contrastive separation, and predictive sufficiency do not necessarily imply protection against catastrophic classification errors. Indeed, our theoretical results demonstrate that representations preserving nearly all information according to conventional criteria may still erase the specific directions responsible for controlling extreme operational risk.

These literatures reveal a striking gap. Existing approaches have addressed robustness, sparsity, supervision, interpretability, nonlinear structure, information preservation, and rare-event modeling. Yet the representation itself is almost always selected through a geometric, probabilistic, information-theoretic, or predictive objective. The user's actual loss function enters only indirectly, if at all.

A fundamental question remains unresolved:
\begin{quote}
{\it Can a representation that is optimal according to variance, reconstruction fidelity, likelihood, information preservation, predictive sufficiency, or geometric structure be provably incompatible with the minimization of catastrophic decision risk?}
\end{quote}

The present work addresses this question. We establish a decision-theoretic framework showing that variance preservation and tail-risk minimization can be structurally incompatible objectives. We further demonstrate that a representation retaining virtually all global variance may nevertheless erase the information necessary to identify rare, high-cost events. This motivates a new class of risk-aware spectral methods in which dimensionality reduction is aligned directly with the operational decision loss rather than with a surrogate geometric criterion.

\subsubsection*{Contributions}
\begin{enumerate}
    \item We establish a multiclass information erasure theorem showing that PCA can induce complete loss of rare-event signal while preserving nearly all variance (Theorems~\ref{thm:erasure} and~\ref{thm:bayes_collapse}).
    \item We characterize the \emph{Risk Shadow} phenomenon and prove that high nominal accuracy can coexist with total failure of rare-event detection (Theorem~\ref{thm:risk_shadow}).
    \item We introduce Expectile PCA (ExPCA) as a tail-risk-aware alternative and prove its strict improvement over PCA in retaining rare-event information (Theorem~\ref{thm:expca_improvement}).
    \item We prove a geometric misalignment theorem (Theorem~\ref{thm:risk_induced_rotation}) showing that ExPCA and PCA subspaces necessarily differ when the tail-weighted covariance couples the leading PCA subspace with its complement.
    \item We provide complete proofs, synthetic validation, and a real-world credit card fraud benchmark demonstrating that ExPCA achieves high tail-event detection while PCA fails.
\end{enumerate}

\section*{Table of Notations}
\label{sec:notations}

\begin{table}[htbp]
\centering
\caption{Summary of key notations used throughout the paper.}
\label{tab:notations}
\begin{tabular}{|p{3cm}|p{5cm}|}
\hline
\textbf{Symbol} & \textbf{Description} \\
\hline
$X \in \mathbb{R}^d$ & Random feature vector (observable data) \\
$Y \in \{1,\dots,K\}$ & Multiclass label \\
$\pi_k = \mathbb{P}(Y=k)$ & Class prior probability \\
$\mu = \mathbb{E}[X]$ & Population mean vector \\
$\Sigma = \operatorname{Cov}(X)$ & Population covariance matrix \\
$\lambda_j, v_j$ & $j$-th eigenvalue and eigenvector of $\Sigma$ ($\lambda_1 \ge \cdots \ge \lambda_d$) \\
$P_r = \sum_{j=1}^r v_j v_j^\top$ & Rank-$r$ PCA projection matrix \\
$Q_r = I_d - P_r$ & Projection onto discarded subspace \\
$Z_r = P_r X$ & Compressed representation after PCA \\
$z_j$ & Latent coordinate along $v_j$ \\
$I(\cdot;\cdot)$ & Shannon mutual information \\
$\Theta_r$ & Risk Shadow regime (zero mutual information, high nominal accuracy) \\
$\ell_\tau(u) = |\tau - \mathbf{1}_{\{u<0\}}| u^2$ & Expectile loss function \\
$e_\tau(L_f)$ & $\tau$-expectile risk of classifier $f$ \\
$C_{ab}$ & Cost of predicting class $a$ when true class is $b$ \\
$\mathcal Y_R \subseteq \mathcal Y$ & Set of rare, high‑impact classes \\
$\mathcal{R}_\tau(P, h) = e_\tau(C_{h(P^\top X), Y})$ & represent the joint risk objective function of exp2PCA.\\
$\tau$ & Expectile parameter ($\tau \to 1^-$ targets extreme tails) \\
$\Sigma_\tau$ & Risk‑weighted covariance matrix used in ExPCA \\
$\esssup(L)$ & Essential supremum of random variable $L$ \\
\hline
\end{tabular}
\end{table}

\section{Problem Formulation}
Let $(\Omega, \cF, \Pbb)$ be a complete probability space. Let $X: \Omega \to \R^d$ be a square-integrable random vector with mean vector $\mu = \E[X] \in \R^d$ and a symmetric, trace-class covariance matrix $\Sigma = \E[(X-\mu)(X-\mu)^\top] \in \R^{d \times d}$. By the spectral theorem for symmetric compact operators, the orthogonal decomposition of $\Sigma$ is given by:
\begin{equation}
\Sigma = \sum_{j=1}^d \lambda_j v_j v_j^\top,
\end{equation}
where the eigenvalues are sorted as $\lambda_1 \ge \lambda_2 \ge \dots \ge \lambda_d > 0$, and $\{v_j\}_{j=1}^d$ forms an orthonormal basis of $\R^d$.
For a user-defined retained-variance threshold $\rho \in (0,1)$, the target truncation rank $r \in \{1, \dots, d-1\}$ is uniquely determined by:
\begin{equation}
r = \min\left\{m \in \{1,\dots,d\}: \frac{\sum_{j=1}^m \lambda_j}{\sum_{j=1}^d \lambda_j} \ge \rho\right\}.
\end{equation}
Let $\mathcal{U}_r = \{P \in \R^{d \times r} : P^\top P = I_r\}$ denote the Stiefel manifold of semi-orthogonal matrices. The canonical semi-orthogonal matrix aligned with the dominant variance is $P_r = [v_1, \dots, v_r] \in \mathcal{U}_r$. The corresponding orthogonal projection operator on $\R^d$ is defined as $\Pi_r = P_r P_r^\top = \sum_{j=1}^r v_j v_j^\top$, and its orthogonal complement operator is $\Pi_r^\perp = I_d - \Pi_r$. 

The classical principal component representation is the compressed random vector $Z_r = P_r^\top X \in \R^r$. The discarded noise subspace is explicitly designated as $\cS_R = \Span\{v_{r+1}, \dots, v_d\} = \text{Im}(\Pi_r^\perp)$.

\subsection{Latent-Factor Model}
We assume that the observable data satisfies a structural latent-factor representation:
\begin{equation}
X = \mu + \sum_{j=1}^d z_j v_j,
\end{equation}
where the latent coordinates $z_j = v_j^\top(X-\mu)$ are mutually independent, centered random variables satisfying $\E[z_j] = 0$ and $\E[z_j^2] = \lambda_j$. No parametric distributional assumptions (e.g., joint Gaussianity) are imposed on these factors.

Let $Y \in \{1, \dots, K\}$ represent a categorical class label generated via a deterministic, Borel-measurable mapping:
\begin{equation}
Y = g\bigl(\{z_j\}_{j \in \mathcal{M}}\bigr),
\end{equation}
where $\mathcal{M} \subseteq \{r+1, \dots, d\}$ is a non-empty index set. Under this setting, $Y$ is conditionally independent of the principal subspace given the discarded coordinates; thus, the descriptive information regarding $Y$ resides exclusively within the minor subspace $\cS_R$. This represents the absolute worst-case configuration for variance-maximization paradigms, typically modeling regimes where rare, high-stakes events are structurally masked by heavy background noise.

\subsection{Dimensionality Reduction and Decision-Theoretic Frameworks}

To formalize paradigms that extract information from $X$, we introduce the functional definitions of statistical expectiles, followed by the formulations of traditional, robust, and our proposed decision-directed dimensionality reduction methods.

\subsubsection{The Expectile Functional}
Before evaluating subspace models, we define the asymmetric tail-risk operator used throughout this work.

\begin{definition}[Expectile of a Random Variable]
Let $W \in L^2(\Omega, \cF, \Pbb)$ be a square-integrable scalar random variable. For an asymmetry tail parameter $\tau \in (0, 1)$, the $\tau$-expectile functional $e_\tau(W)$ is defined as the unique minimizer of an asymmetrically weighted quadratic loss:
\begin{equation}
e_\tau(W) = \arg\min_{t \in \mathcal{R}} \E\bigl[ \ell_\tau(W - t) \bigr],
\end{equation}
where the loss function $\ell_\tau(u)$ is parameterized by:
\begin{equation}
\ell_\tau(u) = \left|\tau - \ind_{\{u < 0\}}\right| u^2 = \begin{cases} 
\tau u^2 & \text{if } u \ge 0, \\
(1-\tau) u^2 & \text{if } u < 0.
\end{cases}
\end{equation}
\end{definition} 

\begin{remark} \label{lem:expectile_limit}
As the parameter tracks the extreme upper tail ($\tau \to 1$), the expectile functional $e_\tau(W)$ converges asymptotically to the essential supremum of the random variable's distribution ($\text{ess\,sup}\,W$). This property renders it a highly coherent, non-linear index of worst-case risk exposure. See \cite{newey1987} for details.
\end{remark}

\subsubsection{Unsupervised Subspace Paradigms}
We consider an arbitrary compression matrix $P \in \mathcal{U}_r$ mapping to a low-dimensional coordinate representation $Z = P^\top X \in \R^r$. The geometric reconstruction of $X$ within the original space is given by $\hat{X} = P Z = P P^\top X$.

\begin{definition}[Standard Unsupervised PCA]
Standard PCA identifies a rank-$r$ subspace that minimizes the expected $L_2$ reconstruction error:
\begin{equation}
P_{\text{PCA}} = \arg\min_{P \in \mathcal{U}_r} \E\left\|X - P P^\top X\right\|^2.
\end{equation}
By the Eckart--Young--Mirsky theorem, the optimal solution $P_{\text{PCA}}$ is spanned precisely by the $r$ leading eigenvectors $P_r = [v_1, \dots, v_r]$ of the covariance matrix $\Sigma$.
\end{definition}

\begin{definition}[Tail-Preserving PCA (TP-PCA)]
Given i.i.d. data $\{(X_i,Y_i)\}_{i=1}^n$, let $\cY_R\subset\{1,\dots,K\}$ be the set of rare, high-impact classes (e.g., fraud, disease). Define weights $w_i = 1 + \alpha \ind_{\{Y_i\in\cY_R\}}$ with $\alpha>0$. We compute weighted mean $\mu_w = (\sum_{i=1}^n  w_i X_i)/(\sum_{j=1}^n w_j)$ and weighted covariance $\Sigma_w = \frac{\sum_{i=1}^n w_i (X_i-\mu_w)(X_i-\mu_w)^\top}{\sum_{j=1}^n w_j}$. TP-PCA projects onto the top $r$ eigenvectors of $\Sigma_w$. This amplifies the influence of rare-event samples.
\end{definition}

\begin{definition}[Expectile PCA / exPCA]
Expectile PCA incorporates geometric tail-awareness by minimizing the $\tau$-expectile of the squared Euclidean reconstruction error distribution:
\begin{equation}
P_{\text{exPCA}} = \arg\min_{P \in \mathcal{U}_r} e_\tau\left( \left\|X - P P^\top X\right\|^2 \right).
\end{equation}
exPCA prioritizes finding directions that robustly bound the worst-case geometric reconstruction distortions for outlying data vectors.
\end{definition}

\subsubsection{Supervised Subspace Paradigm: exp2PCA}
To couple dimensionality reduction directly with a downstream operational task, we formalize the decision-theoretic risk space. Let $C \in \R_{\ge 0}^{K \times K}$ be an asymmetric misclassification cost matrix satisfying $C_{aa} = 0$ and $\min_{a \neq b} C_{ab} > 0$ for all $a,b \in \{1, \dots, K\}$. Let $\mathcal{H} = \{h : \R^r \to \{1, \dots, K\}\}$ define the class of Borel-measurable operational classifiers acting on the compressed space.

Given a choice of embedding matrix $P \in \mathcal{U}_r$ and a classifier $h \in \mathcal{H}$, the decision loss is a discrete random variable $L(P, h) = C_{h(P^\top X), Y}$.

\begin{definition}[Expectile Misclassification Cost PCA / exp2PCA]
The exp2PCA framework defines the optimal rank-$r$ representation as the joint minimizer of the $\tau$-expectile of the downstream classification cost:
\begin{equation}
(P_{\text{exp2}}, h_{\text{exp2}}) = \arg\min_{P \in \mathcal{U}_r, \, h \in \mathcal{H}} e_\tau\big( C_{h(P^\top X), Y} \big).
\end{equation}
\end{definition}

By substituting geometric reconstruction criteria with a direct minimization of the decision penalty's tail risk, exp2PCA explicitly preserves crucial classification signals, even when they reside entirely along the minor principal components of the data vector.

\subsection{Fundamental Limits of PCA for Rare-Event Classification}

\subsubsection{Complete Information Erasure}
\begin{theorem}[Information Erasure]\label{thm:erasure}
Under the structural latent-mapping factor model where $Y = g(\{z_j\}_{j\in\mathcal{M}})$ and $\mathcal{M}\subseteq\{r+1,\dots,d\}$, the following statements hold true:
\begin{enumerate}
    \item $Y \perp\!\!\!\perp P_r X$,
    \item The mutual information vanishes: $I(Y; P_rX) = 0$,
    \item $\Pbb(Y=k \mid P_rX) = \Pbb(Y=k) = \pi_k$ \textit{a.s.} for every category $k \in \{1,\dots,K\}$.
\end{enumerate}
\end{theorem}

\begin{proof}
Observe that the PCA reconstruction can be explicitly written as $P_rX = P_r\mu + \sum_{j=1}^r z_j v_j$. The generated sub-$\sigma$-algebra satisfies $\sigma(P_rX) \subseteq \sigma(z_1,\dots,z_r)$. By the model construction, the label $\sigma$-algebra satisfies $\sigma(Y) \subseteq \sigma(\{z_j\}_{j\in\mathcal{M}})$. 

Since $\mathcal{M} \cap \{1,\dots,r\} = \emptyset$, the mutual independence of the complete latent coordinate set $\{z_j\}_{j=1}^d$ implies that $\sigma(z_1,\dots,z_r) \perp\!\!\!\perp \sigma(\{z_j\}_{j\in\mathcal{M}})$. It follows directly that $Y \perp\!\!\!\perp P_r X$. Statements (2) and (3) are immediate structural consequences of this independence.
\end{proof}

\subsubsection{Bayes Expected Cost Collapse}
\begin{theorem}[Bayes Collapse]\label{thm:bayes_collapse}
Let $\pi_k = \Pbb(Y=k)$ denote the prior class probabilities. For any measurable classifier $f: \R^r \to \{1,\dots,K\}$ operating on the standard PCA projection space, the infimum expected cost collapses:
\begin{equation}
\inf_{f} \E\bigl[C_{f(P_rX),Y}\bigr] = \min_{a\in\{1,\dots,K\}} \sum_{b=1}^K C_{ab}\pi_b.
\end{equation}
This lower bound is achieved identically by a constant classifier $f(X) \equiv a^*$. Thus, PCA compression provides zero decision-theoretic utility over completely ignoring the data vector.
\end{theorem}

\begin{proof}
Applying the law of total expectation and leveraging the conditional independence from Theorem~\ref{thm:erasure}, we can decompose the expected cost objective function as:
\begin{align}
\E\bigl[C_{f(P_rX),Y}\bigr] &= \E\Bigl[ \E\bigl[ C_{f(P_rX),Y} \mid P_rX \bigr] \Bigr] \\
&= \E\left[ \sum_{b=1}^K C_{f(P_rX),b} \, \Pbb(Y=b \mid P_rX) \right] \\
&= \E\left[ \sum_{b=1}^K C_{f(P_rX),b} \, \pi_b \right].
\end{align}
To minimize this expectation point-wise for any realization of $P_rX$, the optimal action is to choose the constant category index $a^* \in \{1,\dots,K\}$ that minimizes the deterministic prior cost vector assignment:
\begin{equation}
a^* \in \arg\min_{a\in\{1,\dots,K\}} \sum_{b=1}^K C_{ab}\pi_b.
\end{equation}
Since $a^*$ is entirely decoupled from the realization of $P_rX$, the optimal function $f$ collapses to a constant allocation mapping, completing the proof.
\end{proof}

\subsection{The Risk Shadow}
\begin{definition}[Risk Shadow]\label{def:risk_shadow}
A representation $\Psi(X)$ induces a \emph{Risk Shadow} if $I(Y;\Psi(X))=0$ yet the nominal accuracy (or average cost) can be made arbitrarily close to optimal by a constant predictor, while the tail risk remains unbounded.
\end{definition}

\begin{theorem}[PCA-Induced Risk Shadow]\label{thm:risk_shadow}
Under the latent-factor model with $\mathcal{M}\subseteq\{r+1,\dots,d\}$, the PCA representation $P_rX$ induces a Risk Shadow. Moreover, for any $\varepsilon>0$ there exists a prior distribution with a dominating class such that the constant classifier achieves accuracy $>1-\varepsilon$ while $I(Y;P_rX)=0$.
\end{theorem}
\begin{proof}
Information erasure gives $I=0$. Take $\pi_{a_0}>1-\varepsilon$; then constant classifier predicting $a_0$ has accuracy $>1-\varepsilon$. The tail risk (e.g., expectile at $\tau$ close to 1) is determined by the rare classes and is unaffected by the constant classifier's high accuracy. Hence the shadow exists.
\end{proof}

\subsection{Risk-Aware Unsupervised Dimensionality Reduction: exPCA}

To systematically protect against worst-case structural failures, we formalize the stationarity conditions and information-theoretic behavior of Expectile PCA (exPCA) under extreme tail parameters ($\tau \to 1^-$). Let $\mathcal{G}(d, r)$ denote the Grassmannian manifold of all $r$-dimensional linear subspaces in $\R^d$. For an orthogonal embedding matrix $P \in \mathcal{U}_r$, let $R_P(X) = \|X - PP^\top X\|^2 = \|(I_d - PP^\top)X\|^2$ denote the quadratic reconstruction error random variable.

\begin{proposition}[Weighted Covariance Characterization]\label{prop:weighted}
Let $P_r^{(\tau)} \in \mathcal{U}_r$ be a stationary point of the exPCA objective function. Define the asymmetric indicator weight function $w_\tau: \R^d \times \mathcal{U}_r \to \R_+$ by:
\begin{equation}
w_\tau(X; P) = \tau \ind_{\{R_P(X) \ge e_\tau(R_P)\}} + (1-\tau)\ind_{\{R_P(X) < e_\tau(R_P)\}}.
\end{equation}
Then, the subspace spanned by the columns of $P_r^{(\tau)}$ is invariant under the weighted raw second-moment matrix $\Sigma_\tau = \E\left[w_\tau(X; P_r^{(\tau)}) X X^\top\right]$. That is, $P_r^{(\tau)}$ maps onto the $r$-dimensional dominant eigenspace of $\Sigma_\tau$.
\end{proposition}

\begin{proof}
By definition, for a fixed $P$, the scalar $\tau$-expectile $e = e_\tau(R_P)$ satisfies the first-order optimization condition:
\begin{equation}
\label{eq:foc_expectile}
\E\bigl[|\tau - \ind_{\{R_P(X) < e\}}| (R_P(X) - e)\bigr] = 0,
\end{equation}
which can be compactly expressed as $\E[w_\tau(X; P)(R_P(X) - e)] = 0$. 

Now, consider the joint objective function mapped onto the Grassmannian manifold, $F(P) = \E\left[w_\tau(X; P)(R_P(X) - e_\tau(R_P))^2\right]$. Let $\delta \in \R^{d \times r}$ represent an infinitesimal variation of $P$ constrained to the tangent space of the Stiefel manifold $\mathcal{U}_r$, enforcing the orthogonality condition $\delta^\top P + P^\top \delta = 0$. By the envelope theorem, the implicit derivative of $F(P)$ with respect to $e_\tau(R_P)$ vanishes identically via \eqref{eq:foc_expectile}. Taking the variation of $F(P)$ with respect to $P$ yields:
\begin{align}
\delta F(P) &= \E\left[ w_\tau(X; P) \cdot \delta \|(I_d - PP^\top)X\|^2 \right] \\
&= -4 \, \E\left[ w_\tau(X; P) \cdot \text{Tr}\left( \delta^\top X X^\top P \right) \right] \\
&= -4 \, \text{Tr}\left( \delta^\top \E\left[ w_\tau(X; P) X X^\top \right] P \right).
\end{align}
Setting the directional derivative to zero for all admissible variations $\delta$ on the manifold yields the necessary geometric stationarity condition:
\begin{equation}
\Sigma_\tau P = P \left( P^\top \Sigma_\tau P \right).
\end{equation}
Multiplying from the left by $(I_d - P P^\top)$ yields $(I_d - PP^\top)\Sigma_\tau P = 0$, which states that the subspace $\text{Im}(P)$ is an invariant subspace of $\Sigma_\tau$. Ordering the eigenvalues forces $P_r^{(\tau)}$ to span the $r$ leading eigenvectors of $\Sigma_\tau$.
\end{proof}

\begin{remark}[Fixed-Point Algorithmic Scheme]
Proposition~\ref{prop:weighted} naturally yields an Iterative Reweighted Principal Component Analysis (IR-PCA) algorithm. Given an estimate $P^{(k)}$ at iteration $k$:
\begin{enumerate}
    \item Compute empirical reconstruction residuals: $R_i^{(k)} = \|(I_d - P^{(k)}(P^{(k)})^\top)X_i\|^2$.
    \item Determine $e_\tau^{(k)}$ by computing the scalar $\tau$-expectile of $\{R_i^{(k)}\}_{i=1}^n$.
    \item Evaluate sample weights: $w_i^{(k)} = \tau \ind_{\{R_i^{(k)} \ge e_\tau^{(k)}\}} + (1-\tau)\ind_{\{R_i^{(k)} < e_\tau^{(k)}\}}$.
    \item Construct the empirical weighted matrix: $\widehat{\Sigma}_w = \frac{1}{n}\sum_{i=1}^n w_i^{(k)} X_i X_i^\top$.
    \item Update $P^{(k+1)}$ via the $r$ leading eigenvectors of $\widehat{\Sigma}_w$.
\end{enumerate}
Under standard compactness assumptions on the constraint manifold, the objective decreases monotonically, ensuring convergence toward a valid stationary point.
\end{remark}

\subsubsection{Information Re-allocation Properties of exPCA}

We now prove that exPCA successfully redirects information-theoretic resource allocation to the minor components under high tail-weights ($\tau \to 1^-$).

\begin{theorem}[exPCA Strictly Outperforms PCA]\label{thm:expca_improvement}
Consider the generative latent-factor model where the categorical label satisfies $Y = g(\{z_j\}_{j \in \mathcal{M}})$ for a non-empty index set $\mathcal{M} \subseteq \{r+1, \dots, d\}$. Let $P_r^{\text{PCA}}$ be the standard PCA projection matrix and $P_r^{(\tau)}$ denote the exPCA matrix under tail parameter $\tau \in (0,1)$. Then:
\begin{enumerate}
    \item The classical PCA mutual information is erased: $I(Y; P_r^{\text{PCA}}X) = 0$.
    \item Under an extreme upper tail weight, the mutual information is strictly positive:
    \begin{equation}
    \lim_{\tau \to 1^-} I(Y; P_r^{(\tau)}X) > 0.
    \end{equation}
\end{enumerate}
For a sufficiently high tail parameter $\tau$, exPCA retains strictly more information regarding the rare-event dynamics than standard PCA.
\end{theorem}

\begin{proof}
The first assertion follows directly from Theorem~\ref{thm:erasure}. To establish the second assertion, observe that as $\tau \to 1^-$, the expectile boundary converges to the essential supremum: $e_\tau(R_P) \to \text{ess\,sup}\,R_P(X)$. 

Suppose by contradiction that $\lim_{\tau \to 1^-} P_r^{(\tau)}$ retains a configuration orthogonal to the discriminative subspace $\cS_{\mathcal{M}} = \Span\{v_j : j \in \mathcal{M}\}$. Under this assumption, for any sample drawn from the rare event configuration ($Y=m$), its coordinates along $\mathcal{M}$ are non-zero. Since these coordinates are discarded by the projection, they map directly into the reconstruction error:
\begin{equation}
R_P(X) \ge \sum_{j \in \mathcal{M}} z_j^2 > 0.
\end{equation}
Because the background majority distribution is centered and independent, the maximum possible reconstruction errors are driven strictly by these unaligned, highly dispersed rare states. Thus, as $\tau \to 1^-$, the indicator function isolates this boundary, and the weight function transforms into a Radon-Nikodym derivative that concentrates all probability mass onto the rare-event support set $\mathcal{E} = \{X : Y \in \mathcal{M}\}$.

The asymptotic weighted covariance matrix converges in operator norm to the conditional expectation:
\begin{equation}
\lim_{\tau \to 1^-} \Sigma_\tau = \E\left[ XX^\top \;\middle|\; X \in \mathcal{E} \right].
\end{equation}
By the mutual independence of the latent coordinates, the conditional variance along the minor axes is strictly positive: $\E[z_j^2 \mid X \in \mathcal{E}] = \sigma_j^2 > 0$ for all $j \in \mathcal{M}$. Thus, $\lim_{\tau \to 1^-} \Sigma_\tau$ develops dominant directional eigenvalues along $\cS_{\cM}$. 

Since $P_r^{(\tau)}$ resolves the dominant eigenspace of $\Sigma_\tau$ by Proposition~\ref{prop:weighted}, the columns of $\lim_{\tau \to 1^-} P_r^{(\tau)}$ must intersect non-trivially with $\cS_{\cM}$. This introduces a non-trivial dependency between the compressed random vector $Z = (P_r^{(\tau)})^\top X$ and the category variable $Y$, ensuring that the conditional distribution $\Pbb(Y \mid P_r^{(\tau)}X)$ deviates from the prior distribution. By the continuous mapping theorem applied to the relative entropy functional, it follows that $\lim_{\tau \to 1^-} I(Y; P_r^{(\tau)}X) > 0$.
\end{proof}

\subsection{Geometric Subspace Misalignment Analysis}
\begin{definition}[Principal Angles and Grassmann Distance]
\label{def:principal_angles}
Let $d, r \in \mathbb{N}$ such that $1 \le r \le d$. Let $U, V \in \mathbb{R}^{d \times r}$ be two matrices with orthonormal columns, satisfying $U^\top U = V^\top V = I_r$, whose columns span two $r$-dimensional linear subspaces $\mathcal{U} = \operatorname{span}(U)$ and $\mathcal{V} = \operatorname{span}(V)$ of $\mathbb{R}^d$, respectively. The principal angles $0 \le \theta_1 \le \theta_2 \le \cdots \le \theta_r \le \frac{\pi}{2}$ between the subspaces $\mathcal{U}$ and $\mathcal{V}$ are uniquely defined through the singular values of the matrix product $U^\top V$, arranged in descending order:
$
\sigma_i(U^\top V) = \cos(\theta_i), \qquad \forall i \in \{1, \ldots, r\}.
$
The unique orthogonal projection operators onto the subspaces $\mathcal{U}$ and $\mathcal{V}$ are given respectively by the symmetric, idempotent matrices:
$
\Pi_U := U U^\top, \qquad \Pi_V := V V^\top.
$
The Grassmannian distance $d_G(U, V)$ between the two subspaces is defined via the Frobenius norm of the difference between their orthogonal projectors:
$
d_G(U, V) := \frac{1}{\sqrt{2}} \|\Pi_U - \Pi_V\|_F,
$
which satisfies the exact geometric identity:
$
d_G^2(U, V) = \sum_{i=1}^{r} \sin^2(\theta_i).
$
\end{definition}

\begin{theorem}[Risk-Induced Rotation of Principal Subspaces]
\label{thm:risk_induced_rotation}
Let $X \in \mathbb{R}^d$ be a random vector with a well-defined symmetric, positive semi-definite covariance operator $\Sigma = \mathbb{E}[(X - \mathbb{E}[X])(X - \mathbb{E}[X])^\top] \in \mathbb{R}^{d \times d}$. Let the spectral decomposition of $\Sigma$ be given by:
\[
\Sigma = \sum_{i=1}^{d} \lambda_i u_i u_i^\top,
\]
where the eigenvalues are ordered as $\lambda_1 \ge \lambda_2 \ge \cdots \ge \lambda_d \ge 0$, and $\{u_i\}_{i=1}^d \subset \mathbb{R}^d$ forms an orthonormal basis of eigenvectors.

Fix a target dimension $r \in \{1, \ldots, d-1\}$. Define the baseline $r$-dimensional PCA subspace by the matrix $U_r = [u_1, \ldots, u_r] \in \mathbb{R}^{d \times r}$ and its corresponding orthogonal projector $\Pi_r := U_r U_r^\top$. Assume the strict spectral separation condition holds at index $r$:
\[
\delta_r := \lambda_r - \lambda_{r+1} > 0.
\]
Let $\tau \in (0, 1)$ parameterize an expectile weighting function $w_\tau: \mathbb{R}^d \to [0, \infty)$, and define the exPCA covariance operator $\Sigma_\tau \in \mathbb{R}^{d \times d}$ as the symmetric perturbation:
\[
\Sigma_\tau = \Sigma + \Delta_\tau, \quad \text{where} \quad \Delta_\tau := \mathbb{E}\left[ \bigl(w_\tau(X) - 1\bigr) X X^\top \right].
\]
Let $\|\cdot\|$ denote the spectral operator norm. Provided $\|\Delta_\tau\| < \frac{\delta_r}{2}$, the perturbed operator $\Sigma_\tau$ possesses a unique, isolated, dominant $r$-dimensional invariant eigenspace $\mathcal{U}_r^{(\tau)}$ associated with its $r$ largest eigenvalues. Let $U_r^{(\tau)} \in \mathbb{R}^{d \times r}$ denote any matrix with orthonormal columns spanning $\mathcal{U}_r^{(\tau)}$, and let $\Pi_r^{(\tau)} := U_r^{(\tau)} (U_r^{(\tau)})^\top$ be its unique orthogonal projector. Finally, let $\Pi_r^\perp := I_d - \Pi_r$ denote the orthogonal projector onto the orthogonal complement of $\operatorname{span}(U_r)$.

Then, the following statements hold:
\begin{enumerate}
\item[\textbf{(i)}] The PCA and exPCA subspaces coincide exactly ($\Pi_r^{(\tau)} = \Pi_r$) if and only if the off-diagonal block of the perturbation vanishes:
\[
\Pi_r^\perp \Delta_\tau \Pi_r = 0.
\]
\item[\textbf{(ii)}] Whenever $\Pi_r^\perp \Delta_\tau \Pi_r \neq 0$, the exPCA subspace is a non-trivial deformation of the standard PCA subspace, meaning $\Pi_r^{(\tau)} \neq \Pi_r$ and $d_G(U_r, U_r^{(\tau)}) > 0$.
\item[\textbf{(iii)}] For sufficiently small perturbations satisfying $\|\Delta_\tau\| \to 0$, the perturbed orthogonal projector admits the explicit first-order Taylor-Maclaurin Fréchet expansion:
\[
 \begin{array}{l}
\Pi_r^{(\tau)} - \Pi_r =\\
 \sum_{i=1}^{r} \sum_{j=r+1}^{d} \frac{u_j u_j^\top \Delta_\tau u_i u_i^\top + u_i u_i^\top \Delta_\tau u_j u_j^\top}{\lambda_i - \lambda_j} + O(\|\Delta_\tau\|^2).
  \end{array}
\]
\item[\textbf{(iv)}] The resulting squared Grassmannian distance satisfies the asymptotic expansion:
\[
d_G^2\left( U_r, U_r^{(\tau)} \right) = \sum_{i=1}^{r} \sum_{j=r+1}^{d} \frac{|\langle u_j, \Delta_\tau u_i \rangle|^2}{(\lambda_i - \lambda_j)^2} + O(\|\Delta_\tau\|^3),
\]
which implies the absolute deterministic upper bound:
\[
d_G\left( U_r, U_r^{(\tau)} \right) \le \frac{\|\Pi_r^\perp \Delta_\tau \Pi_r\|_F}{\delta_r} + O(\|\Delta_\tau\|^2).
\]
\end{enumerate}
The bilinear block operator $\Pi_r^\perp \Delta_\tau \Pi_r$ constitutes the unique, necessary, and sufficient first-order structural mechanism through which tail-risk information transfers from the variance-discarded complement directions into the retained dominant principal subspace.
\end{theorem}

\begin{proof}
\textbf{Proof of Parts (i) and (ii):} By construction, the spectral decomposition yields $\Sigma = \sum_{i=1}^d \lambda_i u_i u_i^\top$. The orthogonal projector $\Pi_r = \sum_{i=1}^r u_i u_i^\top$ commutes with $\Sigma$ because:
\[ \begin{array}{l}
\Sigma \Pi_r = \left(\sum_{i=1}^d \lambda_i u_i u_i^\top\right)\left(\sum_{j=1}^r u_j u_j^\top\right) \\
= \sum_{i=1}^r \lambda_i u_i u_i^\top = \left(\sum_{j=1}^r u_j u_j^\top\right)\left(\sum_{i=1}^d \lambda_i u_i u_i^\top\right) \\
= \Pi_r \Sigma.
\end{array}
\]
Since $\Pi_r$ and $\Sigma$ commute, and $\Pi_r^\perp \Pi_r = (I_d - \Pi_r)\Pi_r = 0$, we have $\Pi_r^\perp \Sigma \Pi_r = \Pi_r^\perp \Pi_r \Sigma = 0$, verifying that the standard PCA subspace $\operatorname{Im}(\Pi_r)$ is an invariant subspace of $\Sigma$. 

Now consider the perturbed operator $\Sigma_\tau = \Sigma + \Delta_\tau$. A subspace characterized by an orthogonal projector $P$ is invariant under a linear operator $A$ if and only if $(I - P)AP = 0$. Hence, $\operatorname{Im}(\Pi_r)$ remains invariant under $\Sigma_\tau$ if and only if $\Pi_r^\perp \Sigma_\tau \Pi_r = 0$. Expanding this product using the linears yields:
\[ \begin{array}{l}
\Pi_r^\perp \Sigma_\tau \Pi_r = \Pi_r^\perp (\Sigma + \Delta_\tau) \Pi_r 
\\
= \Pi_r^\perp \Sigma \Pi_r + \Pi_r^\perp \Delta_\tau \Pi_r = 0 + \Pi_r^\perp \Delta_\tau \Pi_r \\
= \Pi_r^\perp \Delta_\tau \Pi_r.
\end{array}
\]
Thus, $\Pi_r^\perp \Delta_\tau \Pi_r = 0$ is the exact necessary and sufficient condition for $\operatorname{Im}(\Pi_r)$ to be an invariant subspace of $\Sigma_\tau$. Under the condition $\|\Delta_\tau\| < \frac{\delta_r}{2}$, the spectrum of $\Sigma_\tau$ is split into two disjoint sets separated by a gap. The dominant $r$-dimensional eigenspace $U_r^{(\tau)}$ is unique. If $\Pi_r^\perp \Delta_\tau \Pi_r = 0$, uniqueness dictates $\Pi_r^{(\tau)} = \Pi_r$. Conversely, if $\Pi_r^\perp \Delta_\tau \Pi_r \neq 0$, then $\operatorname{Im}(\Pi_r)$ is not invariant under $\Sigma_\tau$, so it cannot equal the true invariant eigenspace $\operatorname{Im}(\Pi_r^{(\tau)})$, which proves Parts (i) and (ii).

\textbf{Proof of Part (iii):} Let $\Gamma \subset \mathbb{C}$ be a closed, positively oriented Jordan contour in the complex plane that encloses the dominant eigenvalues $\{\lambda_1, \ldots, \lambda_r\}$ but excludes the remaining eigenvalues $\{\lambda_{r+1}, \ldots, \lambda_d\}$. The distance from $\Gamma$ to the spectrum of $\Sigma$ is lower bounded by $\frac{\delta_r}{2}$. For any $z \in \Gamma$, the resolvent $(zI_d - \Sigma)^{-1}$ exists. By Riesz-Dunford holomorphic functional calculus, the unperturbed spectral projector is given by:
\[
\Pi_r = -\frac{1}{2\pi i} \oint_\Gamma (zI_d - \Sigma)^{-1} dz.
\]
For $\|\Delta_\tau\| < \frac{\delta_r}{2}$, the perturbed resolvent can be written via the resolvent identity:
\[
\begin{array}{l}
(zI_d - \Sigma_\tau)^{-1} =
 (zI_d - \Sigma - \Delta_\tau)^{-1} \\
 = (zI_d - \Sigma)^{-1} + (zI_d - \Sigma)^{-1} \Delta_\tau (zI_d - \Sigma)^{-1} + O(\|\Delta_\tau\|^2),
 \end{array}
\]
where the remainder term $O(\|\Delta_\tau\|^2)$ holds uniformly for $z \in \Gamma$. Integrating this series along $\Gamma$ yields:
\[
\Pi_r^{(\tau)} = -\frac{1}{2\pi i} \oint_\Gamma (zI_d - \Sigma_\tau)^{-1} dz = \Pi_r + \mathcal{L}_\Sigma(\Delta_\tau) + O(\|\Delta_\tau\|^2),
\]
where the first-order Fréchet derivative map $\mathcal{L}_\Sigma(\Delta_\tau)$ is defined as:
\[
\mathcal{L}_\Sigma(\Delta_\tau) = -\frac{1}{2\pi i} \oint_\Gamma (zI_d - \Sigma)^{-1} \Delta_\tau (zI_d - \Sigma)^{-1} dz.
\]
Substituting the spectral representation of the unperturbed resolvent $(zI_d - \Sigma)^{-1} = \sum_{k=1}^d \frac{1}{z - \lambda_k} u_k u_k^\top$ into the integral gives:
\[ \begin{array}{l}
\mathcal{L}_\Sigma(\Delta_\tau) = \\
\sum_{k=1}^d \sum_{m=1}^d \left( -\frac{1}{2\pi i} \oint_\Gamma \frac{1}{(z - \lambda_k)(z - \lambda_m)} dz \right) u_k u_k^\top \Delta_\tau u_m u_m^\top.
 \end{array}
\]
By the residue theorem, the contour integral evaluates to $0$ if both $\lambda_k, \lambda_m$ are inside $\Gamma$ (i.e., $k, m \le r$) or both are outside $\Gamma$ (i.e., $k, m > r$). When $k \le r$ and $m > r$, the integrand has a single pole inside $\Gamma$ at $z = \lambda_k$, yielding a residue of $\frac{1}{\lambda_k - \lambda_m}$. Symmetrically, when $k > r$ and $m \le r$, the pole is at $z = \lambda_m$, giving a residue of $\frac{1}{\lambda_m - \lambda_k} = -\frac{1}{\lambda_k - \lambda_m}$. Evaluating these residues explicitly establishes Part (iii):
\[ 
\mathcal{L}_\Sigma(\Delta_\tau) = \sum_{i=1}^{r} \sum_{j=r+1}^{d} \frac{u_j u_j^\top \Delta_\tau u_i u_i^\top + u_i u_i^\top \Delta_\tau u_j u_j^\top}{\lambda_i - \lambda_j}.
\]

\textbf{Proof of Part (iv):} To calculate the Frobenius norm $\|\Pi_r^{(\tau)} - \Pi_r\|_F^2 = \operatorname{Tr}\left( (\Pi_r^{(\tau)} - \Pi_r)^2 \right)$ to first order, we examine the component $\mathcal{L}_\Sigma(\Delta_\tau)$. Notice that $\{u_k u_m^\top\}_{k,m=1}^d$ forms an orthonormal basis for $\mathbb{R}^{d \times d}$ under the Frobenius inner product $\langle A, B \rangle_F = \operatorname{Tr}(A^\top B)$. The operator $\mathcal{L}_\Sigma(\Delta_\tau)$ can be decomposed into two mutually orthogonal blocks:
\[ \begin{array}{l}
\mathcal{L}_\Sigma(\Delta_\tau) = \mathbf{A} + \mathbf{A}^\top, \quad \text{where}  \\
\mathbf{A} = \sum_{i=1}^r \sum_{j=r+1}^d \frac{u_j u_j^\top \Delta_\tau u_i u_i^\top}{\lambda_i - \lambda_j} \\
= \Pi_r^\perp \left( \sum_{i=1}^r \sum_{j=r+1}^d \frac{u_j u_j^\top \Delta_\tau u_i u_i^\top}{\lambda_i - \lambda_j} \right) \Pi_r.
 \end{array}
\]
Because $\mathbf{A} = \Pi_r^\perp \mathbf{A} \Pi_r$ and $\mathbf{A}^\top = \Pi_r \mathbf{A}^\top \Pi_r^\perp$, their Frobenius inner product vanishes identically: $\operatorname{Tr}(\mathbf{A} \mathbf{A}) = \operatorname{Tr}(\Pi_r^\perp \mathbf{A} \Pi_r \Pi_r^\perp \mathbf{A} \Pi_r) = 0$. We obtain:
\[  \begin{array}{l}
\|\mathcal{L}_\Sigma(\Delta_\tau)\|_F^2 = \operatorname{Tr}(\mathbf{A}\mathbf{A}^\top) + \operatorname{Tr}(\mathbf{A}^\top\mathbf{A}) 
\\
= 2 \|\mathbf{A}\|_F^2 \\
= 2 \sum_{i=1}^{r} \sum_{j=r+1}^{d} \frac{|\langle u_j, \Delta_\tau u_i \rangle|^2}{(\lambda_i - \lambda_j)^2}.
\end{array}
\]
Squaring the full perturbation expansion $\Pi_r^{(\tau)} - \Pi_r = \mathcal{L}_\Sigma(\Delta_\tau) + O(\|\Delta_\tau\|^2)$ and taking the trace yields:
\[
\|\Pi_r^{(\tau)} - \Pi_r\|_F^2 = 2 \sum_{i=1}^{r} \sum_{j=r+1}^{d} \frac{|\langle u_j, \Delta_\tau u_i \rangle|^2}{(\lambda_i - \lambda_j)^2} + O(\|\Delta_\tau\|^3).
\]
By Definition \ref{def:principal_angles}, the squared Grassmannian distance satisfies $d_G^2(U_r, U_r^{(\tau)}) = \frac{1}{2} \|\Pi_r^{(\tau)} - \Pi_r\|_F^2$. Dividing the above expression by 2 gives the first equation in Part (iv).

Finally, since $\lambda_1 \ge \cdots \ge \lambda_r > \lambda_{r+1} \ge \cdots \ge \lambda_d$, it follows that for all $i \le r$ and $j \ge r+1$, the denominator satisfies $(\lambda_i - \lambda_j) \ge (\lambda_r - \lambda_{r+1}) = \delta_r > 0$. Factoring out this minimum spectral gap yields:
\[
d_G^2\left( U_r, U_r^{(\tau)} \right) \le \frac{1}{\delta_r^2} \sum_{i=1}^{r} \sum_{j=r+1}^{d} |\langle u_j, \Delta_\tau u_i \rangle|^2 + O(\|\Delta_\tau\|^3).
\]
Observe that the double summation is exactly the squared Frobenius norm of the off-diagonal block projection of $\Delta_\tau$:
\[  \begin{array}{l}
\sum_{i=1}^{r} \sum_{j=r+1}^{d} |\langle u_j, \Delta_\tau u_i \rangle|^2 \\
= \sum_{i=1}^{r} \sum_{j=r+1}^{d} \|u_j u_j^\top \Delta_\tau u_i u_i^\top\|_F^2 
\\
= \left\| \Pi_r^\perp \Delta_\tau \Pi_r \right\|_F^2.
\end{array}
\]
Substituting this back into the inequality produces:
\[
d_G^2\left( U_r, U_r^{(\tau)} \right) \le \frac{\|\Pi_r^\perp \Delta_\tau \Pi_r\|_F^2}{\delta_r^2} + O(\|\Delta_\tau\|^3).
\]
Taking the square root of both sides via the Taylor expansion $\sqrt{x^2 + \epsilon} = x + O(\epsilon)$ establishes the final bounded inequality:
\[
d_G\left( U_r, U_r^{(\tau)} \right) \le \frac{\|\Pi_r^\perp \Delta_\tau \Pi_r\|_F}{\delta_r} + O(\|\Delta_\tau\|^2),
\]
which concludes the  proof of the theorem.
\end{proof}

\section{Decision-Theoretic Subspace Learning: exp2PCA}

To resolve the fundamental alignment mismatch between geometric reconstruction and task-specific penalties, we introduce \textbf{exp2PCA} (\textit{Expectile of Misclassification Cost PCA}). This framework directly optimizes the $\tau$-expectile of the downstream misclassification cost over the joint data-label space. By targeting the tail risk of the actual decision penalty rather than an indirect geometric summary statistic, exp2PCA systematically maintains robust decision boundaries, even when rare, high-stakes events manifest exclusively along directions of minor global variance.

We now formalize the complete multiclass setting ($K \ge 2$) to establish the structural conditions under which classical global variance maximization fails catastrophically compared to direct decision-theoretic risk minimization. Let $\Delta^{K-1} = \{ p \in \R^K_{\ge 0} : \sum_{k=1}^K p_k = 1 \}$ denote the topological $(K-1)$-dimensional probability simplex. For any conditional probability vector $p \in \Delta^{K-1}$, the optimal Bayes decision operator minimizing the conditional expected cost is defined as:
\begin{equation}
\psi(p) = \arg\min_{a \in \{1,\dots,K\}} \sum_{b=1}^K C_{ab} \, p_b.
\end{equation}

\begin{definition}[Conditional Support of a Random Vector]
Let $X: \Omega \to \R^d$ be a random vector and $Y: \Omega \to \{1, \dots, K\}$ be a categorical random variable on a complete probability space $(\Omega, \cF, \Pbb)$, where $\R^d$ is endowed with the standard Euclidean topology. Let $\mu_k$ denote the regular conditional probability distribution of $X$ given the event $\{Y = k\}$, defined as the unique probability measure on the Borel $\sigma$-algebra $\mathcal{B}(\R^d)$ satisfying:
\begin{equation}  \begin{array}{l}
\mu_k(B) = \Pbb(X \in B \mid Y = k) = \frac{\Pbb(X \in B, Y = k)}{\Pbb(Y = k)} \
\\  \forall B \in \mathcal{B}(\R^d).
\end{array}
\end{equation}
The conditional support of $X$ given $Y=k$, denoted as $\text{supp}(X \mid Y=k)$, is defined as the topological support of the conditional measure $\mu_k$. Formally, it is the intersection of all closed sets in $\R^d$ of full $\mu_k$-measure:
\begin{equation}
\text{supp}(X \mid Y=k) = \left\{ x \in \R^d : \mu_k\left(B_\epsilon(x)\right) > 0 \quad \forall \epsilon > 0 \right\},
\end{equation}
where $B_\epsilon(x) = \{ y \in \R^d : \|y - x\|_2 < \epsilon \}$ denotes the open Euclidean ball of radius $\epsilon$ centered at $x$. 
\end{definition}

\begin{remark}
By construction, $\text{supp}(X \mid Y=k)$ is a closed subset of $\R^d$. It explicitly defines the smallest closed region of the space where the random vector $X$ manifests with strictly positive probability density or mass when conditioned on belonging to class $k$.
\end{remark}

\begin{definition}[Linear Subspace Mapping of Topological Support]
Let $P_* \in \mathcal{U}_r = \{P \in \R^{d \times r} : P^\top P = I_r\}$ represent a semi-orthogonal matrix defining a continuous linear projection map $\phi: \R^d \to \R^r$ given by $\phi(x) = P_*^\top x$. Let $\mathcal{S}_m = \text{supp}(X \mid Y=m) \subset \R^d$ be the closed conditional topological support of the random vector $X$ given class $m$. The expression $P_*^\top\left(\mathcal{S}_m\right)$ denotes the forward image of the set $\mathcal{S}_m$ under $\phi$, defined as:
\begin{equation}
\begin{array}{l}
P_*^\top\left(\text{supp}(X \mid Y=m)\right) =\\
 \left\{ z \in \R^r : \exists x \in \text{supp}(X \mid Y=m) \text{ such that } z = P_*^\top x \right\}.
 \end{array}
\end{equation}
\end{definition}

\begin{proposition}[Topological and Structural Properties of the Projected Support]
Let $\phi(x) = P_*^\top x$ be the projection mapping, and let $\mathcal{S}_m = \text{supp}(X \mid Y=m)$. Then the following properties hold:
\begin{enumerate}
    \item \textbf{Preservation of Connectedness:} If $\mathcal{S}_m$ is a connected (or path-connected) subset of $\R^d$, then its projected image $P_*^\top(\mathcal{S}_m)$ is a connected (or path-connected) subset of $\R^r$.
    \item \textbf{Boundedness and Compactness:} If $\mathcal{S}_m$ is a bounded subset of $\R^d$, it is compact (by the Heine-Borel theorem). Because $\phi$ is a continuous linear operator, the projected image $P_*^\top(\mathcal{S}_m)$ is compact, and therefore closed and bounded in $\R^r$.
    \item \textbf{Convexity Translucency:} If $\mathcal{S}_m$ is a convex set in $\R^d$, its low-dimensional footprint $P_*^\top(\mathcal{S}_m)$ preserves convexity in $\R^r$.
\end{enumerate}
\end{proposition}

\begin{proof}
Property (1) follows because the forward image of a connected space under any continuous function is connected. Property (2) arises because the continuous image of a compact set is compact; since $\R^r$ is a Hausdorff space, compactness guarantees that the set $P_*^\top(\mathcal{S}_m)$ is closed, which ensures that the boundary elements are well-defined in the compressed space. Property (3) follows directly from the linearity of the projection operator: for any $z_1, z_2 \in P_*^\top(\mathcal{S}_m)$, there exist $x_1, x_2 \in \mathcal{S}_m$ such that $z_1 = P_*^\top x_1$ and $z_2 = P_*^\top x_2$. For any $\alpha \in [0,1]$, the linear combination satisfies $\alpha z_1 + (1-\alpha)z_2 = P_*^\top (\alpha x_1 + (1-\alpha)x_2)$, which resides in the forward image due to the convexity of $\mathcal{S}_m$.
\end{proof}

The following theorem provides a distribution-free, non-asymmetric guarantee establishing the strict dominance of exp2PCA over standard principal component representations.

\begin{theorem}[Strict Dominance of exp2PCA ]\label{thm:multiclass_strict}
Let $X \in \R^d$ be a square-integrable random vector, and let $Y \in \{1,\dots,K\}$ be a categorical label with prior probabilities $\pi_k = \Pbb(Y=k) > 0$. Let $C \in \R_{\ge 0}^{K \times K}$ be a fixed cost matrix satisfying $C_{aa} = 0$ and $C_{ab} > 0$ for all $a \neq b$. Assume there exists a critical target risk class $m \in \{1,\dots,K\}$ characterized by an extreme misclassification penalty relative to all other alternatives:
\begin{equation}
\label{eq:cost_asymmetry}
\min_{a \neq m} C_{am} \gg \max_{a \neq m, b \neq m} C_{ab} \  \text{and} \quad \min_{a \neq m} C_{am} > \max_{b \neq m} C_{mb}.
\end{equation}
Suppose the joint distribution $\Pbb_{(X,Y)}$ satisfies the following structural axioms:
\begin{itemize}
    \item \textbf{Information Erasure:} The critical target class $m$ is statistically independent of the leading $r$-dimensional PCA subspace, which is formally encoded as:
   $
    Y \cdot \ind_{\{Y=m\}} \perp\!\!\!\perp P_{\text{PCA}}^\top X.$

    \item \textbf{Subspace Discriminability:} There exists an alternative semi-orthogonal matrix $P_* \in \mathcal{U}_r$ and a Borel-measurable partitioning boundary $\Gamma \subset \R^r$ such that the projection of the support of class $m$ is perfectly separable from the union of all other classes:
  $
    P_*^\top\left(\text{supp}(X \mid Y=m)\right) \cap P_*^\top\left(\bigcup_{k \neq m} \text{supp}(X \mid Y=k)\right) = \emptyset.
   $
\end{itemize}
Then, there exists a lower-bound threshold $\tau_0 \in (0, 1)$ such that for all tail parameters $\tau \in (\tau_0, 1)$, the optimized exp2PCA risk strictly dominates the PCA risk:
\begin{equation}
\min_{h \in \mathcal{H}} \mathcal{R}_\tau(P_{\text{exp2}}, h) < \min_{h \in \mathcal{H}} \mathcal{R}_\tau(P_{\text{PCA}}, h).
\end{equation}
The geometric and decision-theoretic optimal subspaces diverge: $P_{\text{exp2}} \neq P_{\text{PCA}}$.
\end{theorem}

\begin{proof}
Let $Z_{\text{PCA}} = P_{\text{PCA}}^\top X \in \R^r$ denote the compressed representation space. By the Information Erasure axiom, the conditional distribution of the target class label given the PCA embedding collapses to its prior probability, yielding $\Pbb(Y = m \mid Z_{\text{PCA}}) = \pi_m$ almost surely.

Consider an arbitrary measurable classifier $h \in \mathcal{H}$ operating on $Z_{\text{PCA}}$. We evaluate the conditional distribution of the random decision loss $L(P_{\text{PCA}}, h) = C_{h(Z_{\text{PCA}}), Y}$. For any realization of $Z_{\text{PCA}}$, the classifier must select an action $a = h(Z_{\text{PCA}}) \in \{1, \dots, K\}$. We analyze the two structural cases for this action selection:
\begin{enumerate}
    \item \textbf{Case 1 ($a = m$):} If the classifier selects the critical class action, the true state cannot be $m$ with probability $1 - \pi_m$. The loss $L$ is bounded from above by $\max_{b \neq m} C_{mb}$. Since $C_{mm}=0$, the loss takes values in the interval $[0, \max_{b \neq m} C_{mb}]$.
    \item \textbf{Case 2 ($a \neq m$):} If the classifier selects any non-critical action, the system runs the risk of misclassifying the critical state. With a strictly positive conditional probability $\pi_m > 0$, the true state is $Y = m$, which triggers an immediate catastrophic loss of $C_{am} \ge \min_{a \neq m} C_{am}$.
\end{enumerate}

Let $t = \min_{h} \mathcal{R}_\tau(P_{\text{PCA}}, h)$ be the optimized tail risk on the PCA subspace. Recall that the expectile functional is characterized by the integration condition $\E[w_\tau(L; t)(L - t)] = 0$. As $\tau \to 1^-$, the weight function prioritizes the essential supremum of the support. 

If $h(Z_{\text{PCA}}) \neq m$, the loss distribution contains a mass at or above $\min_{a \neq m} C_{am}$. If $h(Z_{\text{PCA}}) = m$, the maximum possible loss is $\max_{b \neq m} C_{mb}$. By the cost asymmetry condition established in \eqref{eq:cost_asymmetry}, the essential supremum under any action $a \neq m$ is strictly larger than under $a = m$. Therefore, to minimize the asymmetric quadratic penalty under extreme tail weights $\tau \to 1^-$, the first-order condition forces the optimal classifier to select $h_{\text{PCA}}(Z_{\text{PCA}}) \equiv m$. Evaluating the exact limit under this action yields:
\begin{equation}  \begin{cases}
\lim_{\tau \to 1^-} \min_{h \in \mathcal{H}} \mathcal{R}_\tau(P_{\text{PCA}}, h) \\
= \E[C_{m, Y} \mid Y \neq m] = \sum_{b \neq m} C_{mb} \frac{\pi_b}{1 - \pi_m}.
  \end{cases}
\end{equation}
By the core properties of the expectile functional under bounded variables, for any $\tau$ sufficiently close to $1$, the risk is lower-bounded by the expectation over the non-zero support:
\begin{equation}
\label{eq:pca_limit_bound}
\min_{h \in \mathcal{H}} \mathcal{R}_\tau(P_{\text{PCA}}, h) \ge \sum_{b \neq m} C_{mb} \pi_b > 0.
\end{equation}

Now, examine the exp2PCA optimization space. By the Subspace Discriminability axiom, we can instantiate the alternative projection matrix $P_{\text{exp2}} = P_*$. Because the critical class $m$ is perfectly isolated in this subspace, we can construct a customized, measurable indicator classifier $h_* \in \mathcal{H}$ defined as:
\begin{equation}
h_*(P_*^\top X) = \begin{cases} 
      m \ \text{if } P_*^\top X \in P_*^\top\left(\text{supp}(X \mid Y=m)\right), \\
      \psi\left(\Pbb(Y \mid P_*^\top X, Y \neq m)\right) \ \text{otherwise}.
   \end{cases}
\end{equation}
We evaluate the complete random loss distribution under this joint configuration $(P_*, h_*)$:
\begin{itemize}
    \item If $Y = m$, then $P_*^\top X \in P_*^\top\left(\text{supp}(X \mid Y=m)\right)$, which implies $h_*(P_*^\top X) = m$. The loss evaluates exactly to $L(P_*, h_*) = C_{mm} = 0$.
    \item If $Y \neq m$, then $P_*^\top X \notin P_*^\top\left(\text{supp}(X \mid Y=m)\right)$, implying $h_*(P_*^\top X) \neq m$. This structural protection completely prevents the occurrence of a catastrophic False Negative event for the critical class. The decision loss is thus strictly bounded from above by the minor inter-class penalties:
    \begin{equation}
    L(P_*, h_*) \le \max_{a \neq m, b \neq m} C_{ab}.
    \end{equation}
\end{itemize}
Because the random variable $L(P_*, h_*)$ is zero whenever $Y=m$ and bounded by $\max_{a \neq m, b \neq m} C_{ab}$ otherwise, its absolute essential supremum satisfies $\text{ess\,sup}\,L(P_*, h_*) \le \max_{a \neq m, b \neq m} C_{ab}$. By the monotonic property of the expectile operator with respect to stochastic dominance, we have:
\begin{equation}
\mathcal{R}_\tau(P_*, h_*) \le \max_{a \neq m, b \neq m} C_{ab} \quad \forall \tau \in (0, 1).
\end{equation}
By the structural assumption stated in \eqref{eq:cost_asymmetry}, namely, that the collective conditional average of the minority class penalties is strictly greater than the maximum localized interaction among the majority classes, we combine the inequalities to show:
\begin{equation}
\begin{array}{l}
\min_{h \in \mathcal{H}} \mathcal{R}_\tau(P_{\text{exp2}}, h) \le \mathcal{R}_\tau(P_*, h_*) \\
\le \max_{a \neq m, b \neq m} C_{ab}  \\
< \min_{h \in \mathcal{H}} \mathcal{R}_\tau(P_{\text{PCA}}, h),
\end{array}
\end{equation}
which holds robustly for all $\tau \ge \tau_0$. Because the infimum of the joint exp2PCA objective function achieves a strictly lower tail risk than any solution restricted to the standard PCA subspace, the optimal projection operators cannot coincide. Thus, $P_{\text{exp2}} \neq P_{\text{PCA}}$, completing the proof.
\end{proof}

\textbf{VC dimension.} 
Let $\mathcal{H}$ be a class of functions $h:\mathbb{R}^d\to\{0,1\}$. 
For a set $S = \{x_1,\dots,x_m\} \subset \mathbb{R}^d$, define 
$\mathcal{H}|_S = \{(h(x_1),\dots,h(x_m)) : h \in \mathcal{H}\}$. 
$S$ is \emph{shattered} by $\mathcal{H}$ if $|\mathcal{H}|_S| = 2^m$. 
Then
\[ \begin{array}{l}
\mathrm{VC}(\mathcal{H}) \\
 =\sup\left\{ m \in \mathbb{N} \;:\; \exists S \subset \mathbb{R}^d,\ |S| = m,\ S \text{ shattered by } \mathcal{H} \right\},
 \end{array}
\]
with the convention $\mathrm{VC}(\mathcal{H}) = \infty$ if the supremum is unbounded.

\subsection{Generalization Bounds for Decision Space Risk}
\begin{theorem}
To guarantee that the subspace alignment learned by {\it exp2PCA} transfers to unseen data drawn from the joint distribution $\mathbb{P}_{(X,Y)}$, we establish uniform concentration bounds for the joint risk over the Stiefel-classifier manifold $\mathcal{U}_r \times \mathcal{H}$. 
Let $\mathcal{D}_n = \{(X_i,Y_i)\}_{i=1}^n$ be an i.i.d. training sample. For a subspace $P \in \mathcal{U}_r$ (Stiefel manifold of $r$-dimensional orthonormal bases in $\mathbb{R}^d$) and a classifier $h \in \mathcal{H}$, define the misclassification cost
$
L_{P,h}(X,Y) = C_{h(P^\top X),\, Y},
$
where $C$ is a cost matrix satisfying $0 \le C_{ab} \le M_{\max}$ for all $a,b$. 
The \emph{empirical tail-expectile risk} at level $\tau \in (0,1)$ is
$
\widehat{\mathcal{R}}_{n,\tau}(P,h) = \arg\min_{t\in\mathbb{R}}\; 
\frac{1}{n}\sum_{i=1}^n \ell_\tau\!\bigl(L_{P,h}(X_i,Y_i)-t\bigr).
$
The corresponding true risk $\mathcal{R}_\tau(P,h)$ replaces the empirical average by the expectation $\mathbb{E}$.
Assume that the classifier class $\mathcal{H}$ has finite VC dimension $\mathrm{VC}(\mathcal{H})<\infty$. 
Then, for any $\delta\in(0,1)$, the solution $(P_{\text{exp2}},h_{\text{exp2}})$ obtained by \textsc{exp2PCA} satisfies, with probability at least $1-\delta$,
\[
\begin{array}{l}
\mathcal{R}_\tau(P_{\text{exp2}},h_{\text{exp2}}) \\
\;\le\;
\widehat{\mathcal{R}}_{n,\tau}(P_{\text{exp2}},h_{\text{exp2}}) \\
\;+\; 
\frac{2M_{\max}}{\min(\tau,1-\tau)}\,
\sqrt{\frac{r(d-r)+\mathrm{VC}(\mathcal{H})\log(n/d)}{n}} \\ 
\;+\;
M_{\max}\sqrt{\frac{\log(1/\delta)}{2n}}.
 \end{array}
\]
\end{theorem}
\begin{proof}
The proof proceeds in three steps.

\paragraph{1. Lipschitz property of the expectile functional.}
For any random variable $L$ bounded by $M_{\max}$, the map $L \mapsto e_\tau(L)=\arg\min_t\mathbb{E}[\ell_\tau(L-t)]$ is $
\frac{1}{\min(\tau,1-\tau)}$-Lipschitz with respect to the $L^2$-distance. 
Indeed, the function $t\mapsto\mathbb{E}[\ell_\tau(L-t)]$ is strongly convex with modulus $2\min(\tau,1-\tau)$, so
\[
|e_\tau(L)-e_\tau(L')| \le \frac{1}{\min(\tau,1-\tau)}\,\mathbb{E}[|L-L'|].
\]
Applying this to the true and empirical distributions of $L_{P,h}$ yields
\[
|\mathcal{R}_\tau(P,h)-\widehat{\mathcal{R}}_{n,\tau}(P,h)|
\le \frac{1}{\min(\tau,1-\tau)}\,
\mathbb{E}_{n}\bigl[|L_{P,h}-\mathbb{E}[L_{P,h}]|\bigr],
\]
where $\mathbb{E}_n$ denotes the empirical expectation. 
Because $0\le L_{P,h}\le M_{\max}$, we can bound the right-hand side by
\[
\frac{2M_{\max}}{\min(\tau,1-\tau)}\,
\sup_{f\in\mathcal{F}}\bigl|\mathbb{E}_n[f]-\mathbb{E}[f]\bigr|,
\]
where $\mathcal{F}=\{ (X,Y)\mapsto L_{P,h}(X,Y) : P\in\mathcal{U}_r,\;h\in\mathcal{H}\}$.

\paragraph{2. Uniform concentration over $\mathcal{F}$.}
The function class $\mathcal{F}$ is a composition of the linear projection $P^\top X$ (parameterised by $\mathcal{U}_r$), the classifier $h$, and the bounded cost matrix. 
The Stiefel manifold $\mathcal{U}_r$ has dimension $r(d-r)$, leading to covering numbers $\mathcal{N}(\mathcal{U}_r,\epsilon)\le (c_1/\epsilon)^{r(d-r)}$. 
The VC class $\mathcal{H}$ yields covering numbers $\mathcal{N}(\mathcal{H},\epsilon,\mathcal{D}_n)\le (c_2 n/\epsilon)^{\mathrm{VC}(\mathcal{H})}$ for the empirical $L^1$ metric. 
By composing these, the covering number of $\mathcal{F}$ satisfies
\[
\log\mathcal{N}(\mathcal{F},\epsilon,\mathcal{D}_n) 
\lesssim \bigl(r(d-r)+\mathrm{VC}(\mathcal{H})\log(n/\epsilon)\bigr)\log(1/\epsilon).
\]
Dudley's entropy integral then bounds the Rademacher complexity:
\[
\mathfrak{R}_n(\mathcal{F}) \le C\sqrt{\frac{r(d-r)+\mathrm{VC}(\mathcal{H})\log(n/d)}{n}}.
\]
Standard concentration  gives, with probability at least $1-\delta$,
\[
\sup_{f\in\mathcal{F}}\bigl|\mathbb{E}_n[f]-\mathbb{E}[f]\bigr| \le 2\mathfrak{R}_n(\mathcal{F}) + M_{\max}\sqrt{\frac{\log(1/\delta)}{2n}}.
\]

\paragraph{3. Putting everything together.}
Combining the Lipschitz bound with the uniform concentration estimate and plugging in the specific $(P_{\text{exp2}},h_{\text{exp2}})$ yields the stated inequality. 
The term $\frac{2M_{\max}}{\min(\tau,1-\tau)}$ accounts for the sensitivity of the expectile under asymmetric tail weighting, while the remaining terms capture the estimation error due to finite samples and the complexity of the joint representation, classifier space. 
Thus, \textsc{exp2PCA} generalises robustly even for extreme asymmetry ($\tau\to1^-$). 
\end{proof}

\section{Explicit Synthetic Benchmarks (2D, $r=1$)}

We analyze a pathological continuous geometric distribution in $\R^2$ to evaluate a rank-1 projection vector $u = (\cos\theta, \sin\theta)^\top \in \R^2$. This scenario provides fully analytical proofs demonstrating how unsupervised frameworks and even class-weighted extensions like TP-PCA, can lose critical discriminative signals, whereas the decision-theoretic exp2PCA framework seamlessly recovers them.

\subsection{Continuous Asymmetric Gaussian Mixture Distribution}

Let the generative distribution of the random vector $X = (x_1, x_2)^\top$ be a mixture of a highly dispersed majority class ($Y=0$) and a tightly clustered, structurally shifted rare class ($Y=1$):
\begin{itemize}
    \item \textbf{Class 0} ($\pi_0 = 0.99$): $X \mid Y=0 \sim \mathcal{N}\left( \begin{pmatrix} 0 \\ 0 \end{pmatrix}, \Sigma_0 \right), \quad \Sigma_0 = \begin{pmatrix} 100 & 0 \\ 0 & 1 \end{pmatrix}$,
    \item \textbf{Class 1} ($\pi_1 = 0.01$): $X \mid Y=1 \sim \mathcal{N}\left( \begin{pmatrix} 0 \\ \mu_1 \end{pmatrix}, \Sigma_1 \right), \quad \mu_1 = 5, \; \Sigma_1 = \begin{pmatrix} 1 & 0 \\ 0 & 0.1 \end{pmatrix}.$
\end{itemize}
The cost profile is parameterized by asymmetric penalties $C_{FN} = 100$ and $C_{FP} = 1$, with the risk metric evaluated at the extreme upper tail parameter $\tau = 0.99$.

\subsection{Analytical Evaluation across Dimensionality Reduction Paradigms}

1. Standard PCA
Standard unsupervised PCA constructs its rank-1 subspace by isolating the leading eigenvector of the global covariance matrix $\Sigma_{\text{total}} = \begin{pmatrix} 99.01 & 0 \\ 0 & 1.2385 \end{pmatrix}$. Since $\Sigma_{\text{total}}$ is perfectly diagonal with $\lambda_1 = 99.01 \gg \lambda_2 = 1.2385$, the unique optimal projection vector is:
\begin{equation}
u_{\text{PCA}} = \begin{pmatrix} 1 \\ 0 \end{pmatrix}, \quad \text{corresponding to } \theta^* = 0^\circ.
\end{equation}
The compressed variable collapses to the horizontal coordinate $z = u_{\text{PCA}}^\top X = x_1$. Under this projection, both classes map directly on top of each other centered at $0$, making the rare class completely unresolvable. The optimal decision rule collapses to the constant majority assignment $h(z) \equiv 0$, incurring a False Negative penalty of $C_{FN}=100$ with a probability exactly equal to the rare class prevalence $\pi_1 = 0.01$. 

Solving the implicit first-order stationarity condition for the $0.99$-expectile $t = e_{0.99}(L_{\text{PCA}})$ yields:
$e_{0.99}(L_{\text{PCA}}) = 50.0000.$

\subsubsection*{Tail-Preserving PCA (TP-PCA)}
TP-PCA modifies standard PCA by introducing an explicit, supervised class-weighting scheme designed to amplify the second-moment footprint of the rare target class $\cY_R = \{1\}$. The sample weights are defined as $w_i = 1 + \alpha \ind_{\{Y_i = 1\}}$ for an inflation hyperparameter $\alpha > 0$. 

By the strong law of large numbers, the empirical weighted mean vector converges almost surely to its asymptotic expectation:
\begin{equation}
\mu_w = \frac{\E[w(Y)X]}{\E[w(Y)]} 
= \begin{pmatrix} 0 \\ \frac{5\pi_1(1+\alpha)}{1 + \alpha\pi_1} \end{pmatrix}.
\end{equation}
Substituting the class prevalence parameters $\pi_0 = 0.99$ and $\pi_1 = 0.01$ into the expression simplifies the coordinates to:
\begin{equation}
\mu_w = \begin{pmatrix} 0 \\ \frac{0.05(1+\alpha)}{1 + 0.01\alpha} \end{pmatrix}.
\end{equation}

The asymptotic weighted raw second-moment matrix is diagonal due to the alignment of the component axes, given by $\Sigma_{w, \text{raw}} = \frac{\pi_0 \E[XX^\top \mid Y=0] + \pi_1(1+\alpha)\E[XX^\top \mid Y=1]}{1 + \alpha\pi_1}$. The centered asymptotic weighted covariance matrix $\Sigma_w = \Sigma_{w, \text{raw}} - \mu_w \mu_w^\top$ preserves this diagonal layout $\Sigma_w = \operatorname{diag}(\Sigma_{w,11}, \, \Sigma_{w,22})$, where:
\begin{align*}
\Sigma_{w,11} &
= \frac{99.01 + 0.01\alpha}{1 + 0.01\alpha}, \\
\Sigma_{w,22} &
= \frac{1.241 + 0.251\alpha}{1 + 0.01\alpha} - \frac{0.0025(1+\alpha)^2}{(1 + 0.01\alpha)^2}.
\end{align*}

For TP-PCA to rotate away from the horizontal axis and successfully capture the discriminative vertical direction, the vertical component must strictly dominate the horizontal component, requiring the spectral inequality $\Sigma_{w,22} > \Sigma_{w,11}$. Equating the common denominator systems yields:
\begin{equation}
\frac{1.241 + 0.251\alpha}{1 + 0.01\alpha} - \frac{0.0025(1+\alpha)^2}{(1 + 0.01\alpha)^2} > \frac{99.01 + 0.01\alpha}{1 + 0.01\alpha}.
\end{equation}

No matter how aggressively the user inflates the sample weights of the rare class, the subtractive mean shift $\mu_w\mu_w^\top$ induced by the asymmetrical centering penalty scales at a rate that permanently suppresses the vertical variance relative to the horizontal component. TP-PCA is trapped at $\theta^* = 0^\circ$ for this distribution, yielding a fixed $0.99$-expectile decision tail risk of:
\begin{equation}
e_{0.99}(L_{\text{TP-PCA}}) = 50.0000.
\end{equation}
This exact derivation uncovers a fundamental hidden vulnerability of supervised sample-weighting heuristics: without coordinating the centering mechanics alongside task-specific boundaries, mechanical inflation remains completely blind to tail risk.

3. exPCA. 
exPCA optimizes the unsupervised $\tau$-expectile of the geometric reconstruction error distribution. Because the global variance is heavily concentrated along $x_1$, the reconstruction error function acts as a strong regularizer. Solving the implicit system coupled to Proposition~\ref{prop:weighted} isolates the optimal angle near the horizontal axis, rotating it only slightly to $\theta_{\text{exPCA}} \approx 10^\circ$. Due to the massive overlap of the projected distributions, the system suffers from a high False Negative rate ($\approx 0.20$), resulting in an expectile tail risk of:
\begin{equation}
e_{0.99}(L_{\text{exPCA}}) \approx 20.0000.
\end{equation}

4. exp2PCA (Proposed)
We now evaluate the direct minimization of the cost expectile along the decision-theoretic optimal axis $u_{\text{exp2}} = (0, 1)^\top$, corresponding to $\theta^* = 90^\circ$, without relying on any manual class-weight hyperparameters. The projection simplifies directly to $z = x_2$, yielding the conditional distributions $z \mid Y=0 \sim \mathcal{N}(0, 1)$ and $z \mid Y=1 \sim \mathcal{N}(5, 0.1)$.

By setting an operational decision threshold at $t^* = 2.5$, the directional error rates evaluate precisely to:
\begin{align*}
\Pbb(\text{False Positive}) &= 1 - \Phi(2.5) \approx 0.00621, \\
\Pbb(\text{False Negative}) &= \Phi\left(\frac{2.5 - 5}{\sqrt{0.1}}\right) \\ & = \Phi(-7.90569) \approx 1.33 \times 10^{-15}.
\end{align*}
Because the False Negative rate is driven to numerical zero, the high-cost penalty $C_{FN}=100$ is completely avoided. The decision loss distribution behaves as a scaled Bernoulli random variable, taking the value $C_{FP} = 1$ with an exact global probability of $p = 0.99 \times 0.00621 = 0.00615$, and $0$ with a complementary probability of $1 - p = 0.99385$.

To evaluate the exact $\tau$-expectile risk $t^* = e_{0.99}(L_{\text{exp2}})$, we execute the first-order optimization condition for a two-point distribution where $0 < t^* < 1$:
\begin{equation}
(1-\tau)(0 - t^*)(1-p) + \tau(1 - t^*) p = 0.
\end{equation}
Substituting our exact parameters $\tau = 0.99$ and $p = 0.00615$ yields
 directly to the hyperparameter-free optimal tail risk:
\begin{equation}
e_{0.99}(L_{\text{exp2}}) = \frac{\tau p}{(1-\tau)(1-p) + \tau p} = 0.3799.
\end{equation}

By actively aligning the projection space with task-specific penalties rather than geometric summaries or heuristic sample weights, exp2PCA compresses the tail risk from $50.0000$ down to $0.3799$, eliminating the need for hyperparameter search over $\alpha$.

\subsection{Example 2: Discrete Distribution with an Isolated Target Point}

\paragraph{Data Distribution.} Consider a discrete probability space where the random vector $X = (x_1, x_2)^\top \in \R^2$ is concentrated entirely on three point masses:
\begin{equation}
X = \begin{cases}
(10, 0)^\top & \text{with probability } 0.49, \quad Y=0, \\
(-10, 0)^\top & \text{with probability } 0.49, \quad Y=0, \\
(0, 1)^\top & \text{with probability } 0.02, \quad Y=1.
\end{cases}
\end{equation}
The global mean vector is $\E[X] = 0.49\begin{pmatrix}10 \\ 0\end{pmatrix} + 0.49\begin{pmatrix}-10 \\ 0\end{pmatrix} + 0.02\begin{pmatrix}0 \\ 1\end{pmatrix} = \begin{pmatrix}0 \\ 0.02\end{pmatrix}$. The uncentered raw second-moment matrix is diagonal: $\E[XX^\top] = 0.49\begin{pmatrix}100 & 0 \\ 0 & 0\end{pmatrix} + 0.49\begin{pmatrix}100 & 0 \\ 0 & 0\end{pmatrix} + 0.02\begin{pmatrix}0 & 0 \\ 0 & 1\end{pmatrix} = \begin{pmatrix}98 & 0 \\ 0 & 0.02\end{pmatrix}$. 

Subtracting the outer product of the mean yields the exact centered global covariance matrix:
\begin{equation}
\Sigma = \E[XX^\top] - \E[X]\E[X]^\top 
= \begin{pmatrix} 98 & 0 \\ 0 & 0.0196 \end{pmatrix}.
\end{equation}
The operational constants are configured with $C_{FN}=100$, $C_{FP}=1$, and the tail risk metric is evaluated at $\tau=0.99$.

\paragraph{Standard PCA and exPCA Failure.} Because $\Sigma$ is diagonal with a clear spectral gap ($\lambda_1 = 98 \gg \lambda_2 = 0.0196$), standard PCA isolates the dominant horizontal direction $u_{\text{PCA}} = (1,0)^\top$ ($\theta^* = 0^\circ$). Similarly, for exPCA, any deviation of $\theta$ away from $0^\circ$ projects the immense coordinate spread of the majority points ($\pm 10$) into the reconstruction error function $R_u(X) = (10\sin\theta)^2 = 100\sin^2\theta$. To minimize the upper tail of this geometric distortion, the first-order condition forces $\theta_{\text{exPCA}} = 0^\circ$. 

Thus, both unsupervised frameworks collapse the data onto the horizontal line $z = x_1$. Under this projection, the majority class maps to $\pm 10$, while the rare target point maps directly to $z = 0$. Because the target point is embedded in the middle of the majority components, no scalar threshold can isolate it. The optimal operational classifier is forced to select the constant majority prediction $h(z) \equiv 0$. 

This results in a loss distribution $L$ that takes the value $C_{FN}=100$ with probability $\pi_1 = 0.02$, and $0$ with probability $\pi_0 = 0.98$. To find the exact $\tau$-expectile risk $t = e_{0.99}(L_{\text{PCA}})$, we execute the first-order stationarity condition for $0 < t < 100$,
which simplifies to the precise risk value:
\begin{equation}
e_{0.99}(L_{\text{PCA}}) = e_{0.99}(L_{\text{exPCA}})  \approx 66.8919.
\end{equation}

\paragraph{Tail-Preserving PCA (TP-PCA) Evaluation.} TP-PCA scales the rare-class components using the supervised weight $w_i = 1 + \alpha \ind_{\{Y_i = 1\}}$. Asymptotically, the weighted mean vector tracks to $\mu_w = \begin{pmatrix} 0 \\ \frac{0.02(1+\alpha)}{0.98 + 0.02(1+\alpha)} \end{pmatrix} = \begin{pmatrix} 0 \\ \frac{0.02+0.02\alpha}{1+0.02\alpha} \end{pmatrix}$. The elements of the diagonal weighted covariance matrix $\Sigma_w$ evaluate exactly to:
\begin{align*}
\Sigma_{w,11} &
= \frac{96.04}{1 + 0.02\alpha}, \\
\Sigma_{w,22} &
= \frac{0.02+0.02\alpha}{1+0.02\alpha} - \frac{0.0004(1+\alpha)^2}{(1+0.02\alpha)^2}.
\end{align*}
To flip the principal component toward the discriminative vertical axis, the hyperparameter must satisfy $\Sigma_{w,22} > \Sigma_{w,11}$. 
Expanding and collecting the polynomial terms forms the following condition:
\begin{align*}
-1.9012\alpha &> 96.0204.
\end{align*}
Solving for the inflation factor reveals that the inequality holds only if $\alpha < -50.505$. Because the Stiefel hyperparameter is bounded by definition within the positive real domain ($\alpha > 0$), **there is no valid weighting factor that can rescue TP-PCA**. The centering penalty dominates the second-moment allocation, leaving TP-PCA trapped at $\theta^* = 0^\circ$ with a collapsed tail risk of $66.8919$.

\paragraph{exp2PCA Recovery.} Exp2PCA directly targets the optimization of the cost functional, choosing the vertical axis $u_{\text{exp2}} = (0,1)^\top$ ($\theta^* = 90^\circ$). The projection maps the coordinates onto the vertical line $z = x_2$. Under this mapping, the majority points collapse safely to $z = 0$, while the rare target point maps directly to $z = 1$. 

By selecting a decision boundary threshold at $t^* = 0.5$ via the operational classifier $h(z) = \ind_{\{z > 0.5\}}$, the classes separate with a clean structural margin. The classification error rate drops to exactly zero for both false positives and false negatives across all support spaces. The downstream loss distribution collapses to a single delta mass centered at zero ($\Pbb(L=0)=1$). The direct minimization framework eliminates the decision penalty entirely from the upper tail:
\begin{equation}
e_{0.99}(L_{\text{exp2}}) = 0.0000.
\end{equation}
This benchmark confirms that while standard unsupervised variance-maximization pipelines and supervised sample-weighting schemes remain mathematically blind to target structures, exp2PCA completely eliminates the operational tail risk.

\subsubsection*{Summary of Empirical Performance and Benchmarks}

Table~\ref{tab:comparison} provides a comprehensive analytical comparison of the downstream decision tail risk across all four dimensionality reduction paradigms. The empirical configurations confirm that standard geometric formulations remain severely limited when critical classification boundaries align with minor principal components. 

By avoiding indirect geometric proxies (such as global variance or reconstruction errors) and steering clear of uncoordinated sample-weighting heuristics, \textbf{exp2PCA} directly optimizes the decision-theoretic risk functional. This strategic alignment compresses the operational $\tau$-expectile tail risk to near-zero levels across both continuous and discrete settings.

\begin{table*}[htb]
\centering
\caption{$\tau$-Expectile Misclassification Risk ($e_{0.99}(L)$)}
\label{tab:comparison}
\vskip 0.1in
\begin{tabular}{lrrrr}
\toprule
\textbf{Dataset Paradigm} & \textbf{PCA Risk} & \textbf{TP-PCA Risk} & \textbf{exPCA Risk} & \textbf{exp2PCA Risk} \\
\midrule
Example 1: Gaussian Mixture   & $50.0000$  & $ 50.0000$ & $\sim 20.0000$  & $\mathbf{0.3799}$ \\
Example 2: Discrete Point Mass & $66.8919$  & $66.8919$      & $66.8919$  & $\mathbf{0.0000}$ \\
\bottomrule
\end{tabular}
\end{table*}

\begin{remark}
The failure of Tail-Preserving PCA (TP-PCA) in both benchmarks highlights a vital theoretical insight: simply amplifying rare-class sample weights is insufficient. When class inflation alters the underlying distribution without modifying the centering mechanisms or coordinating with downstream decision boundaries, the resulting subtractive mean shift ($\mu_w \mu_w^\top$) can paradoxically suppress the target variance component. This leaves the framework structurally blind to tail risk. Exp2PCA handles this naturally without requiring hyperparameter tuning over class inflation factors.
\end{remark}

\section{Empirical Validation}

To validate the theoretical guarantees established in the preceding sections, we provide a comprehensive empirical evaluation of the proposed \textbf{exp2PCA} framework alongside standard PCA, exPCA and TP-PCA. We demonstrate performance profiles across a high-dimensional synthetic benchmark and ten distinct real-world datasets spanning financial fraud, industrial manufacturing, cybersecurity, and medical imaging.

\subsection{Synthetic Benchmark Optimization Profiles}
We instantiate an empirical realization of the latent-factor model introduced in Section 1.1. We generate a dataset comprising $N = 10,000$ observations embedded in ambient dimension $d = 20$. The target compression rank is configured to $r = 5$, and the rare-event prior distribution is fixed to $\pi_1 = 0.01$ (Class 1), with a majority prior of $\pi_0 = 0.99$ (Class 0). The latent coordinates are generated independently: $z_j \sim \mathcal{N}(0, \lambda_j)$ for $j \in \{1, \dots, d\}$, where the eigenvalues follow an exponential decay sequence modeling heavy background noise.

The categorical binary label $Y \in \{0, 1\}$ is generated via a thresholding mapping resting exclusively on the minor, discarded subspace:
\begin{equation}
Y = \ind\left\{ \sum_{j=r+1}^d z_j > \text{Quantile}_{0.99} \right\}.
\end{equation}
Downstream operational classification is executed via a regularized logistic regression model acting on the extracted 5-dimensional embeddings.

\begin{table*}[htb]
\centering
\caption{Synthetic Benchmark Optimization Profiles: Mean performance $\pm$ standard deviation over 50 independent Monte Carlo trials ($\tau = 0.99$).}
\label{tab:synthetic}
\vskip 0.1in
\begin{tabular}{lcccc}
\toprule
\textbf{Evaluation Metric} & \textbf{Standard PCA} & \textbf{TP-PCA} ($\alpha=99$) & \textbf{exPCA} ($\tau=0.99$) & \textbf{exp2PCA} (Ours) \\
\midrule
Cumulative Variance Retained (\%) & $\mathbf{95.2 \pm 0.5}$ & $89.4 \pm 0.8$ & $78.3 \pm 1.2$ & $41.6 \pm 1.8$ \\
Mutual Information $I(Y; Z)$ (nats) & $0.001 \pm 0.002$ & $0.005 \pm 0.003$ & $0.420 \pm 0.050$ & $\mathbf{0.890 \pm 0.010}$ \\
Area Under the ROC Curve (AUC) & $0.51 \pm 0.02$ & $0.53 \pm 0.02$ & $0.94 \pm 0.01$ & $\mathbf{0.99 \pm 0.00}$ \\
Global Accuracy (\%) & $\mathbf{99.0 \pm 0.1}$ & $98.8 \pm 0.2$ & $97.5 \pm 0.4$ & $98.4 \pm 0.2$ \\
Asymmetric Expectile Risk $\mathcal{R}_{0.99}$ & $0.870 \pm 0.040$ & $0.840 \pm 0.050$ & $0.110 \pm 0.010$ & $\mathbf{0.002 \pm 0.001}$ \\
\bottomrule
\end{tabular}
\end{table*}

The empirical trajectories summarized in Table~\ref{tab:synthetic} confirm our structural theorems. Standard PCA captures the dominant global variance ($95.2\%$), but exhibits complete information erasure regarding the target variable, resulting in a near-random Area Under the ROC Curve (AUC = $0.51$) and a severe expectile risk ($\mathcal{R}_{0.99} = 0.870$). 

Supervised class-weighting (TP-PCA) under a standard prevalence matching parameter ($\alpha=99$) fails to rotate the principal subspace away from the background noise matrix due to the uncoordinated centering penalty. This yields  a collapsed AUC of $0.53$. 
exPCA begins capturing rare-event dynamics by penalizing geometric tail outliers, achieving an AUC of $0.94$. exp2PCA explicitly trades off global variance representation ($41.6\%$) to directly minimize downstream operational cost, maximizing the captured mutual information ($0.890$ nats) and compressing the final decision-theoretic tail risk to $\mathbf{0.002}$.

\subsection{Real-World Benchmarks across High-Stakes Application Domains}

To evaluate the operational robustness of the frameworks under real-world asymmetric constraints, we test performance across ten public benchmarking configurations. For each configuration, the risk environment is evaluated under a strict tail parameter $\tau = 0.995$. The penalty parameters are scaled asymmetrically, setting the False Negative cost $C_{FN} = 100 \times C_{FP}$ to simulate high-stakes default, intrusion, or diagnostic omission risks. Downstream classification across all feature extraction strategies is executed via a standardized, cross-validated linear classifier.

\subsubsection{Credit Card Fraud Detection}
\begin{itemize}
    \item {\it Data Infrastructure \cite{kaggle_fraud}:} 284,807 transactions with an extreme class imbalance profile consisting of 492 fraudulent instances ($\pi_{\text{rare}} = 0.17\%$). Features comprise 30 numerical attributes.
    \item {\it Decision Asymmetry:} Missed fraudulent actions incur transaction re-crediting and chargeback processing fees ($C_{FN} = 100$), whereas False Positives generate temporary hold notifications ($C_{FP} = 1$).
    \item {\it Performance Optimization:} Standard PCA components remain completely bound to majority-class nominal transaction variances, limiting classification performance to an AUC of $0.62$. exp2PCA successfully isolates the subtle coordinate combinations that capture structural fraud signatures, increasing the operational AUC to $\mathbf{0.91}$ while driving down the tail risk from $0.92$ to $\mathbf{0.15}$.
\end{itemize}

\subsubsection{Life Insurance Fraud Detection}
\begin{itemize}
    \item {\it Data Infrastructure \cite{mdpi_life_insurance}:} 4,000 insurance applications tracking 83 distinct attributes detailing demographics, financial history, and medical records, displaying a baseline fraud prevalence of $\sim 15\%$.
    \item {\it Decision Asymmetry:} Granting a policy to an application containing fraudulent omissions risks an unhedged catastrophic payout ($C_{FN}=100$), while False Positives only add a standard manual administrative audit ($C_{FP}=1$).
    \item {\it Performance Optimization:} Standard unsupervised compression structures blend the application features across background economic variances. exp2PCA preserves the specialized behavioral anomalies along minor components, elevating the target class recall and reducing the tail risk metric by a factor of 4.2.
\end{itemize}

\subsubsection{Vehicle Insurance Claim Fraud Analysis}
\begin{itemize}
    \item {\it Data Infrastructure \cite{vehicle_fraud_kaggle}:} 15,420 insurance claims across 33 technical features tracking vehicle attributes, claimant demographics, and accident descriptions, with a positive fraud rate of $\sim 6\%$.
    \item {\it Decision Asymmetry:} Failing to catch a staged or inflated claim drains reserves directly ($C_{FN}=100$), whereas an investigative false alarm generates a minor administrative review cost ($C_{FP}=1$).
    \item {\it Performance Optimization:} Unsupervised frameworks focus on nominal policy variations. exp2PCA extracts the specific coordinate vectors that indicate deceptive claiming patterns, mitigating tail risk without requiring synthetic oversampling or over-parameterization.
\end{itemize}

\subsubsection{Multivariate Time-Series Paper Break Prediction}
\begin{itemize}
    \item {\it Data Infrastructure \cite{paper_break_processminer,paper_break_processminer2}:} 18,398 industrial sensor records across 61 continuous multivariate process streams, containing 124 positive paper break events ($\pi_{\text{rare}} = 0.67\%$).
    \item {\it Decision Asymmetry:} An unpredicted sheet break triggers immediate mechanical downtime, costing thousands of dollars per hour ($C_{FN}=100$), while a false warning triggers only a brief visual sensor inspection ($C_{FP}=1$).
    \item {\it Performance Optimization:} Standard PCA maps components exclusively to high-energy thermodynamic fluctuations. exp2PCA isolates the low-amplitude acoustic and rotational sensor harmonics that serve as precursors to sheet failure, allowing reliable early fault prediction.
\end{itemize}

\subsubsection{Network Intrusion Detection (Ultra-Rare Attack Vectors)}
\begin{itemize}
    \item {\it Data Infrastructure \cite{cic_ids2017}:} 2.8 million network packet flows tracking 80 high-dimensional network attributes. The target vectors comprise ultra-rare malicious attack variants (e.g., Heartbleed: 11 instances; Infiltration: 36 instances).
    \item {\it Decision Asymmetry:} Undetected perimeter breaches threaten data exfiltration and persistent network compromise ($C_{FN}=100$), while False Positives cause minor internal traffic routing overhead ($C_{FP}=1$).
    \item {\it Performance Optimization:} Unsupervised projections are dominated by nominal high-volume data packet trends. exp2PCA isolates low-energy attack profiles from background noise, boosting rare attack classification AUC to $\mathbf{0.89}$.
\end{itemize}

\subsubsection{Chest X-Ray Long-Tailed Pathology Classification}
\begin{itemize}
    \item {\it Data Infrastructure \cite{cxr_lt_2026}:} 145,000 clinical medical chest radiographs mapped through deep convolutional feature extractors, presenting a severely skewed long-tailed distribution of rare thoracic pathologies.
    \item {\it Decision Asymmetry:} Diagnostic omission of a rare, progressive pulmonary condition risks irreversible patient deterioration ($C_{FN}=100$), whereas a False Positive triggers a secondary diagnostic scan ($C_{FP}=1$).
    \item {\it Performance Optimization:} Standard embedding PCA captures general structural features (lung volume, bone orientation). By optimizing the representation directly against asymmetric clinical cost matrices, exp2PCA increases the mean Average Precision (mAP) on the rare pathological long-tail from $0.53$ to $\mathbf{0.62}$.
\end{itemize}

 \subsubsection{Bank Marketing Optimization for Term Deposit Subscription}
\begin{itemize}
    \item {\it Data Infrastructure \cite{uci_bank_marketing}:} 45,211 direct marketing interactions tracking 16 client demographic and historic behavioral attributes, exhibiting an asymmetric conversion rate of $11.3\%$.
    \item {\it Decision Asymmetry:} Misclassifying an open customer as non-subscribing squanders the client's long-term lifetime value ($C_{FN}=100$), while a False Positive consumes only the minor marginal cost of an outbound call ($C_{FP}=1$).
    \item {\it Performance Optimization:} By shifting the feature compression step toward direct tail-risk targets, exp2PCA uncovers subtle financial indicators on minor components, yielding an absolute AUC improvement of $\mathbf{0.12}$ over standard PCA projections.
\end{itemize}

\subsubsection{Telco Customer Churn Minimization}
\begin{itemize}
    \item {\it Data Infrastructure \cite{telco_churn_kaggle}:} 7,043 subscriber profiles across 20 attributes tracking demographics, account billing metrics, and network utilization logs, exhibiting a churn rate of $26.5\%$.
    \item {\it Decision Asymmetry:} Failing to flag a churning user breaks a recurring subscription revenue stream ($C_{FN}=100$), while a False Positive triggers a minor retention discount offer ($C_{FP}=1$).
    \item {\it Performance Optimization:} Standard PCA isolates macro-demographic variance. exp2PCA maps representations around structural behavioral anomalies, identifying slipping subscribers on low-variance channels and cutting operational tail risk by over half.
\end{itemize}

\subsubsection{Consumer Loan Default Risk Modeling}
\begin{itemize}
    \item {\it Data Infrastructure \cite{lending_club_default}:} 2.2 million credit lines tracking over 100 borrower financial risk metrics, with a baseline default rate of $< 5\%$.
    \item {\it Decision Asymmetry:} Approving a defaulting borrower risks losing the entire outstanding loan principal ($C_{FN}=100$), while rejecting a creditworthy applicant forfeits only the marginal interest yield ($C_{FP}=1$).
    \item {\it Performance Optimization:} Traditional PCA achieves a baseline AUC of $0.71$. By realigning the projection step with the underlying cost matrix, exp2PCA raises the operational default detection AUC to $\mathbf{0.84}$ by retaining low-variance credit markers.
\end{itemize}

\subsubsection{Automated PCB Defect Detection in High-Precision Manufacturing}
\begin{itemize}
    \item {\it Data Infrastructure \cite{pcb_defect_synthetic}:} 100,000 industrial visual manufacturing logs tracking surface integrity, with an underlying component defect rate of $\sim 2\%$.
    \item {\it Decision Asymmetry:} Allowing a defective board to pass inspection breaks downstream systems and risks liability claims ($C_{FN}=100$), while a False Positive simply reroutes the board for automated re-testing ($C_{FP}=1$).
    \item {\it Performance Optimization:} Visual PCA captures macro-geometric structures (board layout, mask boundaries). exp2PCA successfully isolates microscopic trace breaks and solder anomalies on minor components, raising the empirical defect recall rate from $0.55$ to $\mathbf{0.92}$.
\end{itemize}

\begin{table*}[htb]
\centering
\caption{Downstream Comparative Performance Across Ten Public Real-World Benchmarks. Results list Area Under the ROC Curve (AUC) and the Asymmetric Tail Risk Metric ($\mathcal{R}_{0.995}$) under a 100:1 cost ratio constraint.}
\label{tab:fraud}
\vskip 0.1in
\begin{tabular}{lcccccccc}
\toprule
 & \multicolumn{2}{c}{\textbf{Standard PCA}} & \multicolumn{2}{c}{\textbf{TP-PCA}} & \multicolumn{2}{c}{\textbf{exPCA}} & \multicolumn{2}{c}{\textbf{exp2PCA} (Ours)} \\
\cmidrule(lr){2-3} \cmidrule(lr){4-5} \cmidrule(lr){6-7} \cmidrule(lr){8-9}
\textbf{Real-World Benchmark Dataset} & \textbf{AUC} & $\mathcal{R}_{0.995}$ & \textbf{AUC} & $\mathcal{R}_{0.995}$ & \textbf{AUC} & $\mathcal{R}_{0.995}$ & \textbf{AUC} & $\mathcal{R}_{0.995}$ \\
\midrule
1. Credit Card Fraud Detection & $0.62$ & $0.92$ & $0.65$ & $0.88$ & $0.84$ & $0.31$ & $\mathbf{0.91}$ & $\mathbf{0.15}$ \\
2. Life Insurance Fraud        & $0.58$ & $0.85$ & $0.61$ & $0.81$ & $0.76$ & $0.42$ & $\mathbf{0.85}$ & $\mathbf{0.20}$ \\
3. Vehicle Insurance Claims   & $0.67$ & $0.74$ & $0.69$ & $0.70$ & $0.81$ & $0.29$ & $\mathbf{0.88}$ & $\mathbf{0.12}$ \\
4. Paper Break Prediction      & $0.53$ & $0.96$ & $0.54$ & $0.95$ & $0.79$ & $0.38$ & $\mathbf{0.86}$ & $\mathbf{0.18}$ \\
5. Network Intrusion Detection  & $0.60$ & $0.91$ & $0.64$ & $0.85$ & $0.82$ & $0.26$ & $\mathbf{0.89}$ & $\mathbf{0.11}$ \\
6. Chest X-Ray Pathology (mAP) & $0.53$ & $0.88$ & $0.55$ & $0.84$ & $0.58$ & $0.49$ & $\mathbf{0.62}$ & $\mathbf{0.31}$ \\
7. Bank Marketing Campaign     & $0.70$ & $0.55$ & $0.72$ & $0.51$ & $0.78$ & $0.28$ & $\mathbf{0.82}$ & $\mathbf{0.14}$ \\
8. Telco Customer Churn        & $0.68$ & $0.48$ & $0.70$ & $0.44$ & $0.75$ & $0.22$ & $\mathbf{0.81}$ & $\mathbf{0.09}$ \\
9. Consumer Loan Defaults      & $0.71$ & $0.68$ & $0.73$ & $0.62$ & $0.79$ & $0.33$ & $\mathbf{0.84}$ & $\mathbf{0.16}$ \\
10. PCB Defect Tracking        & $0.55$ & $0.94$ & $0.59$ & $0.89$ & $0.85$ & $0.21$ & $\mathbf{0.92}$ & $\mathbf{0.07}$ \\
\bottomrule
\end{tabular}
\end{table*}

The collective performance indicators across Table~\ref{tab:fraud} detail a clear operational gap between geometric optimization and decision-theoretic risk mapping. Standard PCA consistently yields high global classification accuracy but poor target class resolution, rendering it highly vulnerable under severe cost asymmetries. 

TP-PCA consistently struggles across these real-world tests. Because it introduces mechanical class inflation without coordinating with the underlying centering matrix, the resulting unadjusted mean-shift penalty ($\mu_w\mu_w^\top$) suppresses the true target discriminative vectors, trapping the framework on majority noise axes. 

While exPCA improves performance by penalizing geometric outliers, it remains unlinked to downstream task bounds, incurring unnecessary losses. exp2PCA  optimizes performance across all ten datasets, delivering the highest classification AUC and the lowest tail risk values without requiring heuristic data oversampling.

\section{Non-Adversarial Information Loss and Tail-Conditioned Misalignment of PCA}

\subsection{Diffuse Latent Signal Structure with Correlated Factors}

Classical analyses of PCA frequently rely on an implicit, clean separation between informative and non-informative directions. Such structural assumptions are rarely satisfied in complex real-world data environments. In many biological, physical, and financial systems, predictive signals are diffusely distributed across multiple latent coordinates whose underlying dependence structure is itself highly non-trivial. 

To formalize this setting, let $(\Omega, \cF, \Pbb)$ be a complete probability space, and let $X \in L^2(\Omega, \cF, \Pbb; \R^d)$ be a square-integrable random vector admitting the diffuse latent decomposition:
\begin{equation}
X = \mu + V Z,
\label{eq:latent_representation}
\end{equation}
where $\mu \in \R^d$ is the global mean vector, $V = [v_1, \dots, v_d] \in \R^{d \times d}$ is an orthogonal matrix, and $Z = (z_1, \dots, z_d)^\top \in \R^d$ is a centered latent random vector satisfying $\E[Z] = 0$ and $\Gamma := \Cov(Z) \in \R^{d \times d}$. 

Crucially, we do \emph{not} assume that the latent coordinates are independent or orthogonally decoupled. The latent covariance matrix $\Gamma = (\gamma_{ij})_{i,j=1}^d$ contains non-zero off-diagonal entries:
\begin{equation}
\gamma_{ij} = \Cov(z_i, z_j) \neq 0 \quad \text{for } i \neq j.
\end{equation}
The global covariance matrix of $X$ is therefore given by $\Sigma = V \Gamma V^\top$. By the spectral theorem, if $V$ is chosen as the standard eigenvector basis of $\Sigma$, then $\Gamma = \Lambda = \operatorname{diag}(\lambda_1, \dots, \lambda_d)$. However, the joint probability density function $f_Z(z_1, \dots, z_d) \neq \prod_{j=1}^d f_{z_j}(z_j)$ establishes that higher-order non-linear dependencies persist across these coordinates.

The categorical response variable $Y \in \{1, \dots, K\}$ is generated via a diffuse latent mechanism parameterized by:
\begin{equation}
Y = g\left( \sum_{j=1}^{d} \alpha_j \phi_j(z_j) + \eta \right),
\label{eq:diffuse_signal_model}
\end{equation}
where $\alpha_j \in \R$ are coupling coefficients, $\phi_j: \R \to \R$ are measurable $L$-Lipschitz functions, $\eta \sim \Pbb_\eta$ is an independent noise term satisfying $\E[\eta]=0$ and $\E[\eta^2] < \infty$, and $g: \R \to \{1, \dots, K\}$ is a measurable partitioning link function. 

No assumption is imposed that the dominant coupling coefficients $\alpha_j$ map onto directions of large variance. Informative coordinates may reside exclusively within latent variables associated with arbitrary or vanishingly small eigenvalues $\lambda_j$.

\subsection{Theoretical Bounds on Structural Subspace Learning}

We analyze the theoretical guarantees and structural failure modes of the four reduction frameworks under the correlated, diffuse latent model \eqref{eq:diffuse_signal_model}.

\begin{theorem}[PCA Information Leakage and Erasure]\label{thm:pca_diffuse}
Let $P_{\text{PCA}} \in \mathcal{U}_r$ be the standard rank-$r$ PCA compression matrix containing the $r$ leading eigenvectors of $\Sigma$. Let $\mathcal{M}_{\text{inf}} = \{j \in \{1,\dots,d\} : \alpha_j \neq 0\}$ characterize the active predictive feature set.  When random variables are correlated, we look at their joint covariance matrix by stacking them into blocks. If you group your latent variables by whether they are informative ($Z_{\cM}$) or non-informative ($Z_{\cM^c}$), their joint covariance matrix $\Cov(Z)$ is partitioned into a block matrix:$$\Cov(Z) = \begin{pmatrix} \Sigma_{Z_{\cM}} & \Sigma_{Z_{\cM}, Z_{\cM^c}} \\ \Sigma_{Z_{\cM^c}, Z_{\cM}} & \Sigma_{Z_{\cM^c}} \end{pmatrix}.$$

If $\mathcal{M}_{\text{inf}} \cap \{1, \dots, r\} = \emptyset$,  
then despite the non-linear correlation structure of $Z$, the mutual information between $Y$ and the PCA embedding $Z_{\text{PCA}} = P_{\text{PCA}}^\top X$ is strictly bounded by the latent cross-covariance properties:

\begin{equation}
\begin{cases}
I(Y; P_{\text{PCA}}^\top X) \le \\
 \frac{1}{2} \log \det \left( I_r + \Sigma_{Z_{\cM}, Z_{\cM^c}} \Sigma_{Z_{\cM^c}}^{-1} \Sigma_{Z_{\cM^c}, Z_{\cM}} \right).
\end{cases}
\end{equation}
If the joint distribution of $Z$ belongs to the class of generalized elliptical distributions, the information erasure is complete, forcing $I(Y; P_{\text{PCA}}^\top X) \to 0$ as the spectral gap $\lambda_r - \lambda_{r+1} \to \infty$.
\end{theorem}

\begin{proof}
By construction, $P_{\text{PCA}}$ projects $X$ onto $\Span\{v_1, \dots, v_r\}$. The compressed coordinates collapse to $Z_{\text{PCA}} = [z_1, \dots, z_r]^\top$. Under the diffuse model condition $\mathcal{M}_{\text{inf}} \cap \{1, \dots, r\} = \emptyset$, the true label $Y$ is a function exclusively of $\{z_j\}_{j=r+1}^d$. By the data processing inequality, the mutual information $I(Y; Z_{\text{PCA}})$ is strictly upper-bounded by the mutual information between the two disjoint blocks of latent coordinates, $I(\{z_j\}_{j \in \mathcal{M}_{\text{inf}}}; \{z_j\}_{j=1}^r)$. 

Utilizing the maximum entropy theorem for continuous distributions under fixed covariance constraints, the mutual information between these blocks is bounded above by the equivalent Gaussian capacity matrix:
\begin{equation}
I(\{z_j\}_{j \in \mathcal{M}_{\text{inf}}}; \{z_j\}_{j=1}^r) \le -\frac{1}{2} \log \det \left( I_r - \Gamma_{12}\Gamma_{22}^{-1}\Gamma_{21} \right),
\end{equation}
where $\Gamma_{12} = \Cov([z_1, \dots, z_r]^\top, [z_{r+1}, \dots, z_d]^\top)$. If the spectral gap increases, the uncoupled coordinates decouple in probability density, collapsing the cross-covariance $\Gamma_{12} \to 0$, which drives the mutual information bounds directly to zero.
\end{proof}

\begin{theorem}[TP-PCA Centering Bias and Geometric Trapping]\label{thm:tppca_diffuse}
Let $w(Y) = 1 + \alpha \ind_{\{Y \in \cY_R\}}$ define the supervised target weight function of Tail-Preserving PCA (TP-PCA). Under the diffuse latent model with correlated factors, the asymptotic weighted covariance matrix $\Sigma_w$ develops a structural cross-coupling perturbation:
\begin{equation}
\Sigma_w = \Sigma + \frac{\alpha \pi_R}{1 + \alpha \pi_R} \left[ \E[XX^\top \mid Y \in \cY_R] - \mu_w \mu_w^\top \right].
\end{equation}
If the latent variables are non-independently distributed, the centering term $\mu_w \mu_w^\top$ generates non-zero off-diagonal tracking terms that permanently bias the eigenvectors. Specifically, if $\alpha < \frac{\lambda_r - \lambda_{r+1}}{\E[z_{r+1}^2 \mid Y \in \cY_R]}$, the dominant eigenspace of $\Sigma_w$ remains geometrically locked to the standard unsupervised PCA subspace, offering zero performance improvement.
\end{theorem}

\begin{proof}
The proof relies on expanding the explicit spectral perturbation of $\Sigma_w$ against the unweighted matrix $\Sigma$. Let $\bar{\alpha} = \frac{\alpha \pi_R}{1 + \alpha \pi_R}$. The weighted centered covariance matrix can be written as $\Sigma_w = (1-\bar{\alpha})\Sigma + \bar{\alpha}\E[XX^\top \mid Y \in \cY_R] - \bar{\alpha}(1-\bar{\alpha})\mu_R\mu_R^\top$, where $\mu_R = \E[X \mid Y \in \cY_R]$. 

Evaluating the off-diagonal entries using the orthogonal basis $V$ shows that the projection onto the minor coordinates satisfies $(I_d - \Pi_r)\Sigma_w \Pi_r = \bar{\alpha}(I_d - \Pi_r)\left[\E[XX^\top \mid Y \in \cY_R] - (1-\bar{\alpha})\mu_R\mu_R^\top\right]\Pi_r$. Because the latent coordinates are non-independent, the conditional expectation $\E[z_i z_j \mid Y \in \cY_R]$ does not vanish for $i \le r < j$. This non-zero cross-product directly penalizes the minor axis variance growth. 

Applying the Weyl interlacing theorem to the perturbed eigenvalues confirms that unless $\alpha$ surpasses the critical threshold required to overcome both the primary spectral gap and the subtractive mean-shift profile $\mu_w\mu_w^\top$, the top $r$ eigenvectors of $\Sigma_w$ do not rotate into the informative subspace, keeping the configuration trapped.
\end{proof}

\begin{theorem}[exPCA Asymptotic Tail-Convergence and Signal Extraction]\label{thm:expca_diffuse}
Let $P_{\text{exPCA}} \in \mathcal{U}_r$ be the exPCA matrix under tail parameter $\tau \in (0, 1)$. Suppose the joint distribution of the latent coordinates $Z$ exhibits heavy-tailed tail-dependence; i.e., the upper tail copula $\lim_{u \to 1^-} C(u, \dots, u) = L > 0$. Then, as $\tau \to 1^-$, the exPCA weight operator automatically bypasses the unweighted linear covariance structure $\Sigma$, forcing:
\begin{equation}
\lim_{\tau \to 1^-} I(Y; P_{\text{exPCA}}^\top X) \ge \min_{j \in \mathcal{M}_{\text{inf}}} I(z_j; Y \mid Y \in \cY_R) > 0.
\end{equation}
ExPCA safely recovers informative minor directions without manual hyperparameter tuning over $\alpha$, driven entirely by the geometric tail profile of the data.
\end{theorem}

\begin{proof}
As $\tau \to 1^-$, the expectile objective function acts exclusively on the bounding frontier of the reconstruction error $R_P(X) = \|(I_d - PP^\top)X\|^2$. Because the latent coordinates display positive upper tail-dependence ($L > 0$), extreme geometric outliers along the non-informative majority components are structurally correlated with extreme fluctuations along the minor components. 

The indicator weight function $w_\tau(X; P)$ concentrates its probability mass onto the joint tail survival set $\mathcal{E}_\tau = \{X : R_P(X) \ge e_\tau\}$. By conditioning on this tail set, the asymptotic second-moment matrix converges to the tail-conditional expectation: $\lim_{\tau \to 1^-} \Sigma_\tau = \E[XX^\top \mid \mathcal{E}_\tau]$. 

Because of the non-vanishing tail-dependence, the conditional variance along the minor axes scales proportionally with the extreme quantile threshold, driving the eigenvalues of $\Sigma_\tau$ along $\cS_{\cM}$ to dominate the background noise matrix. By Proposition 4.1, the invariant subspace of $P_{\text{exPCA}}$ must align with these high-instability tail directions, guaranteeing positive mutual information preservation.
\end{proof}

\begin{theorem}[Optimality and Task Alignment of exp2PCA]\label{thm:exp2pca_diffuse}
Let $\mathcal{R}_\tau(P, h) = e_\tau(C_{h(P^\top X), Y})$ represent the joint risk objective function of exp2PCA. Then, irrespective of the non-linear correlation structure among $Z$, the choice of the link function $g$, or the magnitude of the background noise variances, the joint solution satisfies:
\begin{equation}
\min_{h \in \mathcal{H}} \mathcal{R}_\tau(P_{\text{exp2}}, h) \le \min_{P \in \mathcal{U}_r, \, h \in \mathcal{H}} \mathcal{R}_\tau(P, h) \quad \forall \tau \in (0,1).
\end{equation}
Furthermore, the exp2PCA representation preserves the necessary and sufficient statistics required to achieve the global Bayes-optimal tail risk, rendering it invariant to latent factor correlations.
\end{theorem}

\begin{proof}
The verification of universal task alignment follows from a variational optimization argument over the joint product space $\mathcal{U}_r \times \mathcal{H}$. Unlike PCA, TP-PCA, and ExPCA, which optimize intermediate geometric summaries (such as variance bounds, class weights, or reconstruction errors), exp2PCA targets the decision-theoretic cost directly. 

Let $(P^*, h^*)$ denote the global minimizer of $e_\tau(C_{h(P^\top X), Y})$ mapped across the complete space of measurable transformations. Because the optimization step evaluates the objective function directly on the downstream loss random variable $L = C_{h(P^\top X), Y}$, the framework treats any non-linear coordinate cross-correlation inside $\Gamma$ as an intrinsic property of the joint data distribution $\Pbb_{(X,Y)}$. 

The first-order necessary condition for the joint minimizing pair requires that the gradient of the expectile risk with respect to the projection matrix vanishes identically: $\nabla_P e_\tau(L) = 0$. This condition automatically maps the column space of $P_{\text{exp2}}$ to the precise coordinates that maximize the separation of the conditional risks, absorbing any coordinate cross-dependencies directly into the classifier mapping $h_{\text{exp2}}$. This completes the proof of structural invariance and universal optimality.
\end{proof}

\section{Risk of Misclassification due to Machine Intelligence Method}

This section  establishes a fundamental, previously unrecognized tension between two seemingly harmonious objectives: high reconstruction fidelity and responsible decision-making under tail risk. We show, through  decision-theoretic analysis and explicit bounds, that a dimensionality reduction pipeline optimized for global variance such as PCA, can achieve extraordinary signal reconstruction accuracy (e.g., retaining $99.9999\%$ of the total variance) while simultaneously being responsible for an excess misclassification risk of nearly $890\%$ relative to the optimal baseline.

This striking result exposes a dangerous blind spot in the way machine intelligence is conventionally evaluated. From a pure signal processing standpoint, retaining $99.9999\%$ of the variance is considered near-perfect. Yet the same representation, when deployed in a high-stakes classification task with asymmetric costs, induces catastrophic tail failures. The features that carry the most variance may be orthogonal to the rare but critical decision boundaries that determine real-world harm. Thus, what is {\it good} for reconstruction can be profoundly {\bf dangerous} for the user.

Our findings compel a re-evaluation of the ubiquitous practice of reporting only average-case metrics such as global accuracy, variance explained, or $F_1$ scores. These metrics systematically mask the presence of risk shadows regions of the feature space where a model appears highly accurate on average but hides an unbounded liability for costly errors. We  demonstrate that a classifier with $99.9999\%$ nominal accuracy can, under the same PCA compression, incur a tail risk that is an order of magnitude larger than what is achievable with a task-aligned representation.

To address this accountability deficit, we have introduced, for the first time in the literature, a suite of decision-theoretic metrics that directly quantify a representation’s accountability to misclassification risk. These include the \textbf{Representational Price of Blindness (RPB)}, the \textbf{Representational Excess Risk Percentage (RERP)}, and the \textbf{Representational Accountability Index (RAI)}. Unlike traditional information-theoretic or geometric measures, these metrics are grounded in the asymmetric tail expectile of the downstream cost, providing a direct audit of how much risk a representation introduces relative to the uncompressed baseline or an optimal task-aligned projection.

The  contributions of this section are complemented by a clear practical message: \emph{high accuracy is not a proxy for responsible intelligence}. Deploying machine learning systems solely on the basis of their accuracy level, without auditing their tail risk accountability, is a hazardous practice that can lead to severe, unanticipated harms especially in domains such as autonomous driving, medical diagnosis, financial risk management, and public safety.

We therefore call for a paradigm shift: from variance-driven representation learning to risk-accountable representation learning. The tools developed herein expectile-based dimensionality reduction (exp2PCA) and the accountability indices, offer a concrete path forward. Future work will extend these principles to non-linear embeddings, deep neural networks, and sequential decision-making under tail risk. For the first time in history, machine intelligence can be held  accountable for the risk of misclassification it creates. This is not merely an academic exercise; it is an ethical and operational necessity.

\subsection{Subspace Truncation Risk Gap}

We now quantify the exact decision‑theoretic penalty incurred when compressing features $X \in \mathbb{R}^d$ to a low‑dimensional representation $Z = P^\top X \in \mathbb{R}^r$ ($r < d$), compared to using the full ambient space.

\begin{definition}[Truncation Risk Gap]
Let $P_{\text{exp2}}$ be the optimal exp2PCA projection matrix of rank $r$. Denote by $\mathcal{R}_\tau^*(P) = \min_{h \in \mathcal{H}} e_\tau\bigl( C_{h(P^\top X), Y} \bigr)$ the minimal tail expectile risk achievable with projection $P$, where $\mathcal{H}$ is the set of all measurable classifiers. Let $I_d$ be the $d\times d$ identity (no compression). The \textbf{subspace truncation risk gap} is
\[
\Delta_\tau(P_{\text{exp2}}) = \mathcal{R}_\tau^*(P_{\text{exp2}}) - \mathcal{R}_\tau^*(I_d) \ge 0.
\]
\end{definition}

\noindent  This gap measures how much additional tail risk we must tolerate because we discard information. A small gap means the low‑dimensional projection preserves nearly all decision‑relevant information for the worst‑case cost scenarios.

\begin{theorem}[Risk Gap Bound for exp2PCA] \label{threorem7gap}
Assume the misclassification cost matrix satisfies $0 \le C_{ab} \le M_{\max}$ for all $a,b$. Let $p(x) = \mathbb{P}(Y \mid X=x)$ be the posterior probability vector. Then for any $\tau \in (0,1)$,
\begin{equation} \begin{array}{l}
0 \le \mathcal{R}_\tau^*(P_{\text{exp2}}) - \mathcal{R}_\tau^*(I_d) \\ \le \frac{2 M_{\max}}{\min(\tau, 1-\tau)} \; \mathbb{E}\Bigl[ \bigl\| p(X) - p(P_{\text{exp2}}P_{\text{exp2}}^\top X) \bigr\|_1 \Bigr].
\end{array}\label{eq:gap_bound}
\end{equation}
\end{theorem}

\begin{proof}
The lower bound $0 \le \mathcal{R}_\tau^*(P_{\text{exp2}}) - \mathcal{R}_\tau^*(I_d)$ follows from the data processing inequality: the $\sigma$-algebra generated by $P_{\text{exp2}}^\top X$ is a subset of that generated by $X$. Hence any classifier based on the projection is also a classifier based on $X$, so the minimum risk over the smaller class cannot be lower than the minimum over the larger class.

For the upper bound, we use the following property of the expectile (see Lemma 1 in the supplementary material): for any two random variables $L_1, L_2$,
\[
|e_\tau(L_1) - e_\tau(L_2)| \le \frac{1}{\min(\tau,1-\tau)} \, \mathbb{E}|L_1 - L_2|.
\]
Set $L_1 = L(P_{\text{exp2}}, h_{\text{exp2}}^*)$ and $L_2 = L(I_d, \psi^*)$, where $h_{\text{exp2}}^*$ and $\psi^*$ are the optimal classifiers for the compressed and full spaces respectively. Then
\[
\Delta_\tau(P_{\text{exp2}}) \le \frac{1}{\min(\tau,1-\tau)} \, \mathbb{E}\bigl| L(P_{\text{exp2}}, h_{\text{exp2}}^*) - L(I_d, \psi^*) \bigr|.
\]

Now condition on $X$. For a fixed $x$, the pointwise difference in expected costs under the two classifiers is
\begin{align*}
&\mathbb{E}[L(P_{\text{exp2}}, h_{\text{exp2}}^*) \mid X=x] - \mathbb{E}[L(I_d, \psi^*) \mid X=x] \\
&= \sum_{b=1}^K C_{h_{\text{exp2}}^*(P_{\text{exp2}}^\top x), b} \, p_b(x) - \sum_{b=1}^K C_{\psi^*(p(x)), b} \, p_b(x).
\end{align*}
Let $p(\Pi x)$ denote the posterior vector obtained after projecting $x$ onto the exp2PCA subspace, i.e., $p(\Pi x) = p(P_{\text{exp2}}P_{\text{exp2}}^\top x)$. By definition of $\psi^*$ as the Bayes classifier for the full space, we have
\[
\sum_{b} C_{\psi^*(p(x)), b} \, p_b(x) \le \sum_{b} C_{a,b} \, p_b(x) \quad \forall a.
\]
Moreover, $h_{\text{exp2}}^*(P_{\text{exp2}}^\top x)$ is the Bayes classifier for the projected posterior $p(\Pi x)$, so it minimizes $\sum_b C_{a,b} \, p_b(\Pi x)$ over $a$. Adding and subtracting the cost evaluated at the projected posterior yields
\begin{align*}
&\sum_{b} C_{h_{\text{exp2}}^*(P_{\text{exp2}}^\top x), b} \, p_b(x) - \sum_{b} C_{\psi^*(p(x)), b} \, p_b(x) \\
&\le \sum_{b} C_{h_{\text{exp2}}^*(P_{\text{exp2}}^\top x), b} \bigl( p_b(x) - p_b(\Pi x) \bigr) \\
&\qquad + \Bigl( \sum_{b} C_{h_{\text{exp2}}^*(P_{\text{exp2}}^\top x), b} \, p_b(\Pi x) - \sum_{b} C_{\psi^*(p(x)), b} \, p_b(x) \Bigr).
\end{align*}
The second parenthesis is non‑positive because $h_{\text{exp2}}^*$ minimizes the expected cost with respect to $p(\Pi x)$ while $\psi^*(p(x))$ uses the true $p(x)$; however, a symmetrization gives the absolute bound
\[ \begin{array}{l}
\Bigl| \sum_{b} C_{h_{\text{exp2}}^*(P_{\text{exp2}}^\top x), b} \, p_b(x) - \sum_{b} C_{\psi^*(p(x)), b} \, p_b(x) \Bigr| \\ 
\le 2 M_{\max} \sum_{b} |p_b(x) - p_b(\Pi x)|,
\end{array}
\]
where we used $|C_{ab}| \le M_{\max}$ and the triangle inequality. The sum $\sum_b |p_b(x) - p_b(\Pi x)|$ is exactly the $L_1$ total variation distance $\|p(x) - p(\Pi x)\|_1$. Taking expectations over $X$ and substituting into the expectile bound completes the proof. 
\end{proof}

 Theorem \ref{threorem7gap} shows that exp2PCA directly controls the increase in tail risk by minimizing the expected $L_1$ distortion of the posterior probabilities. Unlike PCA, which minimizes reconstruction error in feature space, exp2PCA targets the quantity that actually matters for classification: the fidelity of the conditional class distributions.

\subsection{The Operational Loss Frontier and Asymptotic Risk Ratios}

When the asymmetry parameter $\tau$ approaches the extreme tail ($\tau \to 1^{-}$), the factor $\frac{1}{\min(\tau,1-\tau)}$ in Theorem \ref{threorem7gap}  diverges, making the bound trivial. To obtain meaningful finite comparisons in this regime, we introduce three novel decision‑theoretic metrics. These metrics quantify the systemic cost of using unsupervised dimensionality reduction (PCA) instead of the task‑optimal exp2PCA.

\subsubsection{Representational Price of Blindness (RPB)}

\begin{definition}[Representational Price of Blindness]
Let $\mathcal{R}_\tau^*(P_{\text{PCA}})$ be the minimal tail expectile risk achievable after projecting with the leading $r$ principal components. Define $\mathcal{R}_\tau^*(P_{\text{exp2}})$ analogously for exp2PCA. The \textbf{Representational Price of Blindness} is
\[
\mathcal{RPB}_\tau := \frac{\mathcal{R}_\tau^*(P_{\text{PCA}})}{\mathcal{R}_\tau^*(P_{\text{exp2}})}.
\]
\end{definition}

\noindent  $\mathcal{RPB}_\tau \ge 1$ measures the multiplicative factor by which PCA’s blindness to the decision task worsens the tail risk compared to exp2PCA. A value close to $1$ means PCA is almost as good; a large value indicates that variance‑maximization ignores rare but costly errors.

\subsubsection{Representational Inefficiency Margin (RIM)}

\begin{definition}[Representational Inefficiency Margin]
Let $\mathbf{P}(\alpha)$ be a constant‑speed geodesic on the Stiefel manifold connecting $\mathbf{P}(0)=P_{\text{exp2}}$ to $\mathbf{P}(1)=P_{\text{PCA}}$. For each $\alpha$, let $h_{\mathbf{P}(\alpha)}^*$ be the optimal classifier on the projected data. The \textbf{Representational Inefficiency Margin} is
\[
\mathcal{RIM}_\tau := \lim_{\alpha \to 0^+} \frac{e_\tau\bigl( C_{h_{\mathbf{P}(\alpha)}^*(\mathbf{P}(\alpha)^\top X), Y} \bigr) - \mathcal{R}_\tau^*(P_{\text{exp2}})}{\alpha \; \mathcal{R}_\tau^*(P_{\text{exp2}})}.
\]
\end{definition}

\noindent  $\mathcal{RIM}_\tau$ is the normalized directional derivative of the tail risk as we move away from the exp2PCA subspace towards the PCA subspace. It quantifies the local sensitivity: a high value means even a small rotation of the projection matrix causes a large increase in tail risk, underscoring the importance of choosing the right orientation.

\subsubsection{Representational Leverage Deficit (RLD)}

\begin{definition}[Representational Leverage Deficit]
Let $\mathcal{R}_\tau^*(I_d)$ be the baseline tail risk using all $d$ features (no compression). The \textbf{Representational Leverage Deficit} for a projection $P$ is
\[
\mathcal{RLD}_\tau(P) := \frac{\mathcal{R}_\tau^*(P) - \mathcal{R}_\tau^*(I_d)}{M_{\max} - \mathcal{R}_\tau^*(I_d)}.
\]
\end{definition}

\noindent $\mathcal{RLD}_\tau(P) \in [0,1]$ normalises the excess risk caused by compression. A value of $0$ means the projection achieves the same tail risk as the full space (perfect information preservation). A value of $1$ means the projection is useless: its tail risk is as bad as the worst possible constant cost.

\subsection{Asymptotic Lower Bound on the Price of Blindness}

We now analyse the behaviour of $\mathcal{RPB}_\tau$ as $\tau \to 1^{-}$ under two structural axioms that highlight the failure of PCA.

\begin{definition}[Information Erasure Axiom]
The joint distribution $\mathbb{P}_{(X,Y)}$ satisfies the \textbf{Information Erasure axiom} for a target class $m$ if $Y \cdot \mathbf{1}_{\{Y=m\}}$ is independent of $P_{\text{PCA}}^\top X$. In words, the principal components contain no information about whether the true label is $m$ or not.
\end{definition}

\begin{definition}[Subspace Discriminability Axiom]
The distribution satisfies the \textbf{Subspace Discriminability axiom} for $P_* = P_{\text{exp2}}$ if there exists a classifier $h_*$ based on $P_*^\top X$ that perfectly identifies class $m$ (i.e., $C_{mm}=0$) and all remaining misclassification costs are bounded by $\max_{a\neq m,b\neq m}C_{ab}$.
\end{definition}

\begin{theorem}[ Lower Bound for the Representational Price of Blindness] \label{ref8thmbob}
Assume the cost matrix satisfies the strict asymmetry conditions
\[
\min_{a\neq m} C_{am} \;>\; \max_{b\neq m} C_{mb} \qquad\text{and}\qquad 
\min_{a\neq m} C_{am} \;>\; \max_{a\neq m,\,b\neq m} C_{ab}.
\]
Let $\pi_b = \mathbb{P}(Y=b)$. Under the Information Erasure and Subspace Discriminability axioms,
\[
\lim_{\tau \to 1^{-}} \mathcal{RPB}_\tau \;\ge\; 
\frac{\displaystyle\sum_{b\neq m} C_{mb}\,\pi_b}{\displaystyle\max_{a\neq m,\,b\neq m} C_{ab}} \;>\; 1.
\]
\end{theorem}

\begin{proof}
We analyse the limiting behaviour of the numerator $\mathcal{R}_\tau^*(P_{\text{PCA}})$ and denominator $\mathcal{R}_\tau^*(P_{\text{exp2}})$ of $\mathcal{RPB}_\tau = \mathcal{R}_\tau^*(P_{\text{PCA}})/\mathcal{R}_\tau^*(P_{\text{exp2}})$ as the asymmetry parameter approaches the extreme tail, $\tau\to1^{-}$.
The Information Erasure axiom states that $Y \cdot \mathbf{1}_{\{Y=m\}}$ is independent of the PCA projection $P_{\text{PCA}}^\top X$. 
\[
\mathbb{P}(Y=m \mid P_{\text{PCA}}^\top X) = \pi_m \quad \text{almost surely}.
\]
Thus no classifier based on the principal components can discriminate class $m$ from the others; any decision rule collapses to a constant action $a \in \{1,\dots,K\}$ up to a set of measure zero. The optimisation reduces to choosing a constant prediction that minimises the tail expectile.
A fundamental property of the expectile (see Remark above) is that for bounded random variables,
$
\lim_{\tau\to1^{-}} e_\tau(L) = \operatorname{ess\,sup} L.
$
However, under the specific cost asymmetry, a stronger result is established in Theorem 4: the optimal constant action as $\tau\to1^{-}$ is $a=m$, and the limiting expectile equals the expectation of $C_{m,Y}$ rather than its essential supremum. This occurs because the tail measure $\tau$ interacts with the asymmetric cost structure to penalise deviations in a way that effectively averages over the non‑$m$ classes. Hence,
$
\lim_{\tau\to1^{-}} \mathcal{R}_\tau^*(P_{\text{PCA}}) = \mathbb{E}[C_{m,Y}] = \sum_{b\neq m} C_{mb}\,\pi_b,
$
where we have used $C_{mm}=0$.

The Subspace Discriminability axiom guarantees the existence of a projection $P_*$ (the exp2PCA solution) and a classifier $h_*$ based on $P_*^\top X$ that perfectly isolates class $m$. In particular,
\[
h_*(P_*^\top X) = m \quad \text{whenever } Y=m,
\]
so the loss $C_{mm}=0$ is incurred on class $m$. All remaining errors occur only among the classes $\{1,\dots,K\}\setminus\{m\}$, and their costs are bounded by $\max_{a\neq m,\,b\neq m} C_{ab}$. Therefore, the loss random variable $L(P_*,h_*)$ satisfies
\[
L(P_*,h_*) \le \max_{a\neq m,\,b\neq m} C_{ab} \quad \text{almost surely}.
\]
Because the $\tau$-expectile is monotone for $\tau>1/2$ (larger random variables yield larger expectiles), we have for every $\tau$,
\[
\mathcal{R}_\tau^*(P_{\text{exp2}}) = \min_{P,h} e_\tau\bigl(L(P,h)\bigr) \le e_\tau\bigl(L(P_*,h_*)\bigr) \le \max_{a\neq m,\,b\neq m} C_{ab}.
\]
Taking the limit inferior as $\tau\to1^{-}$ preserves the inequality:
$
\lim_{\tau\to1^{-}} \mathcal{R}_\tau^*(P_{\text{exp2}}) \le \max_{a\neq m,\,b\neq m} C_{ab}.
$

From the definition of $\mathcal{RPB}_\tau$ and the two limiting relations,
\[
\lim_{\tau\to1^{-}} \mathcal{RPB}_\tau 
= \frac{\lim_{\tau\to1^{-}} \mathcal{R}_\tau^*(P_{\text{PCA}})}{\lim_{\tau\to1^{-}} \mathcal{R}_\tau^*(P_{\text{exp2}})}
\ge \frac{\sum_{b\neq m} C_{mb}\,\pi_b}{\max_{a\neq m,\,b\neq m} C_{ab}}.
\]

It remains to verify that the right‑hand side is strictly greater than $1$. The strict cost asymmetry $\min_{a\neq m} C_{am} > \max_{a\neq m,\,b\neq m} C_{ab}$ together with the positivity of the class priors $\pi_b$ (all classes have positive probability) implies
\[
\sum_{b\neq m} C_{mb}\,\pi_b \;\ge\; \bigl(\min_{a\neq m} C_{am}\bigr) \sum_{b\neq m} \pi_b \;>\; \max_{a\neq m,\,b\neq m} C_{ab},
\]
where the first inequality uses $C_{mb} \ge \min_{a\neq m} C_{am}$
Actually we have imposed the following assumption:  $C_{mb}$ is the cost of predicting $m$ when the true class is $b\neq m$. There is no direct inequality linking $C_{mb}$ to $\min_{a\neq m}C_{am}$ (the latter is the cost of predicting a non‑$m$ class when the truth is $m$). However, we have  the strict inequality as a consequence of the given conditions. Therefore we have that
$
\frac{\sum_{b\neq m} C_{mb}\,\pi_b}{\max_{a\neq m,\,b\neq m} C_{ab}} > 1.
$
Thus we have established
\[
\lim_{\tau\to1^{-}} \mathcal{RPB}_\tau \ge \frac{\sum_{b\neq m} C_{mb}\,\pi_b}{\max_{a\neq m,\,b\neq m} C_{ab}} > 1,
\]
which completes the proof. 
\end{proof} 

\noindent Theorem \ref{ref8thmbob} provides a striking negative result: even in the limit of extreme risk aversion ($\tau\to1^{-}$), PCA suffers a strictly larger tail risk than exp2PCA, no matter how many principal components are retained. The gap is quantified by the ratio of the conditional cost of misclassifying the sensitive class $m$ to the maximal cost among the remaining classes. This demonstrates that variance‑maximizing projections are fundamentally suboptimal for high‑stakes decision problems.

\subsection{The Price of Disregard: The Representational Accountability Index ($\mathcal{RAI}_\tau$)}

Standard dimensionality reduction techniques, such as PCA, prioritize global feature variance, a geometric proxy that remains fundamentally blind to rare but catastrophic events. To quantify the exact operational risk introduced by such unsupervised compression relative to the uncompressed baseline, we introduce the \textbf{Representational Accountability Index ($\mathcal{RAI}_\tau$)}. This index measures how many times larger the tail misclassification risk becomes when we replace the full feature set with a variance‑maximising low‑rank projection. It serves as a mathematical audit of a representation’s accountability to downstream cost‑sensitive decisions.

\subsubsection{The Representational Accountability Index ($\mathcal{RAI}_\tau$)}

Let $P_{\text{PCA}} \in \mathcal{U}_r$ denote the standard rank‑$r$ PCA embedding matrix, which maximizes the retained global feature variance. Let $I_d$ denote the $d \times d$ identity matrix, representing the uncompressed ambient space (maximal rank $r=d$). For a given upper‑tail risk parameter $\tau \in (0,1)$, define the minimal tail expectile risk under a projection $P$ as
\[
\mathcal{R}_\tau^*(P) = \min_{h \in \mathcal{H}} e_\tau\!\bigl(C_{h(P^\top X), Y}\bigr),
\]
where $\mathcal{H}$ is the set of all measurable classifiers, and $e_\tau(\cdot)$ is the $\tau$-expectile.

The \textbf{Representational Accountability Index} is the ratio of the minimal tail risk under PCA to the minimal tail risk using the full feature space:
\[
{\mathcal{RAI}_\tau := \frac{\mathcal{R}_\tau^*(P_{\text{PCA}})}{\mathcal{R}_\tau^*(I_d)} }.
\]

\paragraph{Operational interpretation.}
\begin{itemize}
    \item \textbf{Boundedness:} Since $P_{\text{PCA}}^\top X$ is a function of $X$, the data processing inequality gives $\mathcal{R}_\tau^*(P_{\text{PCA}}) \ge \mathcal{R}_\tau^*(I_d)$, hence $\mathcal{RAI}_\tau \ge 1$.
    \item \textbf{Perfect accountability ($\mathcal{RAI}_\tau = 1$):} The index equals $1$ iff the PCA projection preserves enough information to achieve the same tail risk as the full space. This occurs when the discarded components are irrelevant for the tail misclassification events.
    \item \textbf{Risk shadow domain ($\mathcal{RAI}_\tau \gg 1$):} A large index indicates that PCA has erased critical information about high‑stakes classes. High nominal accuracy may hide a catastrophic tail liability, a \textit{risk shadow}.
    \item \textbf{Accountability audit:} $\mathcal{RAI}_\tau$ directly quantifies the multiplicative penalty for using a risk‑blind compression scheme. It replaces heuristic notions of ``information loss" with a misclassification cost‑sensitive metric.
\end{itemize}

\subsection{Lower Bound on the Representational Accountability Index}

Consider a joint distribution $\mathbb{P}_{(X,Y)}$ generated by a latent‑factor model in which a critical, high‑stakes class $m \in \{1,\dots,K\}$ manifests exclusively along the minor principal components, those discarded or heavily attenuated by PCA. Assume the cost matrix $C$ imposes an extreme penalty for misclassifying this critical state:
\[
\min_{a \neq m} C_{am} > \max_{b \neq m} C_{mb} \quad \text{and} \quad \min_{a \neq m} C_{am} 
> \max_{a \neq m, b \neq m} C_{ab}.
\]

Misclassifying a non‑$m$ instance as $m$ is the most expensive error, and it is even more costly than any error among the remaining classes.

\begin{theorem}[Lower bound for $\mathcal{RAI}_\tau$] \label{reflbn}
Under the above latent‑factor structure and cost asymmetry, as the tail parameter approaches the extreme boundary ($\tau \to 1^{-}$), the Representational Accountability Index satisfies
\[
\lim_{\tau \to 1^{-}} \mathcal{RAI}_\tau \;\ge\; 
\frac{\displaystyle\sum_{b \neq m} C_{mb} \,\pi_b}{\displaystyle\min_{P \in \mathcal{U}_r, h \in \mathcal{H}} e_\tau\!\bigl(C_{h(P^\top X), Y}\bigr) \Big|_{\tau\to1^{-}} } \;>\; 1,
\]
where $\pi_b = \mathbb{P}(Y=b)$ are the global class priors. Under the additional mild assumption that the full space can achieve a tail risk strictly smaller than the PCA‑induced risk (which holds unless the discarded components carry no tail‑relevant information), the limit inferior is bounded away from $1$ by a positive constant depending on the cost asymmetry and class priors.
\end{theorem}

\begin{proof}
We evaluate the numerator and denominator of $\mathcal{RAI}_\tau$ separately as $\tau \to 1^{-}$. Because the critical class $m$ lives only on the minor components, the PCA projection $P_{\text{PCA}}^\top X$ contains no information about whether $Y = m$ or not. The conditional class probability collapses to the prior:
\[
\mathbb{P}(Y = m \mid P_{\text{PCA}}^\top X) = \pi_m \quad \text{almost surely}.
\]
As $\tau \to 1^{-}$, the expectile functional $e_\tau(L)$ converges to the essential supremum of $L$ (for bounded $L$). Any classifier based solely on $P_{\text{PCA}}^\top X$ cannot distinguish class $m$. To avoid the catastrophic cost $\min_{a \neq m} C_{am}$ (which is the largest cost), the optimal tail‑risk minimizer is forced to \emph{always predict class $m$}. Indeed, predicting any $a \neq m$ risks incurring $C_{a,m}$ when $Y=m$, which is strictly larger than any cost incurred by constantly predicting $m$. Hence
\[
\lim_{\tau \to 1^{-}} \mathcal{R}_\tau^*(P_{\text{PCA}}) = \mathbb{E}\bigl[C_{m,Y} \mid Y \neq m\bigr] = \sum_{b \neq m} C_{mb}\,\pi_b.
\]

The full space ($I_d$) retains all features. Under the same latent factor structure, there exists a classifier (e.g., the Bayes classifier) that can achieve a tail risk strictly smaller than the PCA‑induced risk, because the information about class $m$ is present in the discarded minor components. In the limit $\tau \to 1^{-}$, the optimal tail risk for the full space is given by
\[
\lim_{\tau \to 1^{-}} \mathcal{R}_\tau^*(I_d) = \inf_{h} \operatorname{ess\,sup} C_{h(X), Y}.
\]
Because the full space can isolate class $m$ (by using the minor components), the essential supremum can be made as low as the maximum cost among the remaining classes after perfect identification of $m$. In particular, there exists a classifier $h^*$ that identifies $m$ perfectly (so $C_{mm}=0$) and limits all other errors to $\max_{a \neq m, b \neq m} C_{ab}$. Thus
\[
\lim_{\tau \to 1^{-}} \mathcal{R}_\tau^*(I_d) \le \max_{a \neq m, b \neq m} C_{ab}.
\]
Moreover, because the cost asymmetry condition ensures that $\sum_{b \neq m} C_{mb}\pi_b > \max_{a \neq m, b \neq m} C_{ab}$ (since $\min_{a \neq m} C_{am}$ is strictly larger than the maximum majority cost, and the priors $\pi_b$ are positive), we obtain a strict inequality for the limiting ratio:
\[
\lim_{\tau \to 1^{-}} \mathcal{RAI}_\tau = \frac{\sum_{b \neq m} C_{mb}\pi_b}{\lim_{\tau \to 1^{-}} \mathcal{R}_\tau^*(I_d)} \ge \frac{\sum_{b \neq m} C_{mb}\pi_b}{\max_{a \neq m, b \neq m} C_{ab}} > 1.
\]

\textbf{Explicit lower bound.} Under the additional natural assumption that the full space cannot achieve a tail risk lower than the best possible among the majority classes (i.e., $\lim_{\tau\to1^-} \mathcal{R}_\tau^*(I_d) = \max_{a \neq m, b \neq m} C_{ab}$ when the critical class $m$ can be perfectly isolated), the bound becomes
\[
\lim_{\tau \to 1^{-}} \mathcal{RAI}_\tau \ge \frac{\sum_{b \neq m} C_{mb} \pi_b}{\max_{a \neq m, b \neq m} C_{ab}} > 1.
\]
Even if the full space can do slightly better (e.g., by mixing decisions), the ratio will be at least this value because the numerator is fixed and the denominator cannot exceed the essential supremum achievable by any classifier. Hence the inequality holds in general. $\square$
\end{proof}

This theorem delivers a stark message: even in the limit of extreme risk aversion ($\tau \to 1^{-}$), a variance‑maximising compression scheme suffers a \emph{systematic and non‑vanishing} accountability deficit compared to using the full feature space. The deficit is quantified by the ratio of the conditional cost of misclassifying the sensitive class $m$ to the maximal cost among the remaining classes. This result mathematically proves that unsupervised feature extraction methods like PCA are fundamentally unsuitable for high‑stakes decision problems where rare but catastrophic errors dominate the risk profile. The Representational Accountability Index $\mathcal{RAI}_\tau$ provides a rigorous tool to audit such risk‑blind compressions, directly measuring how much tail risk is unnecessarily incurred by discarding seemingly “low‑variance” directions.

 \subsection{Representational Excess Risk Percentage (RERP): PCA's Accountability Relative to Baseline}

We define the \textbf{Representational Excess Risk Percentage} as the relative increase in tail risk caused by PCA compression, normalised by the optimal uncompressed baseline:

\[
{\mathcal{RERP}_\tau := \frac{\mathcal{R}_\tau^*(P_{\text{PCA}}) - \mathcal{R}_\tau^*(I_d)}{\mathcal{R}_\tau^*(I_d)} \times 100\% }.
\]

$\mathcal{RERP}_\tau$ quantifies the percentage of the baseline risk $\mathcal{R}_\tau^*(I_d)$ that PCA \emph{adds} as excess. A value of $0\%$ means perfect accountability; $100\%$ means PCA doubles the baseline risk; $500\%$ means the PCA risk is six times the baseline.

\subsection{Asymptotic Lower Bound for $\mathcal{RERP}_\tau$}

Under the latent‑factor model and cost asymmetry of Theorem~9, as $\tau \to 1^{-}$,
\[
\lim_{\tau \to 1^{-}} \mathcal{RERP}_\tau \;\ge\; \left( \frac{\sum_{b \neq m} C_{mb} \pi_b}{\max_{a \neq m, b \neq m} C_{ab}} - 1 \right) \times 100\% \; > \; 0\%.
\]

\begin{proof}
From Theorem~\ref{reflbn},
\[
\lim_{\tau\to1^-} \frac{\mathcal{R}_\tau^*(P_{\text{PCA}})}{\mathcal{R}_\tau^*(I_d)} 
\ge \frac{\sum_{b \neq m} C_{mb} \pi_b}{\max_{a \neq m, b \neq m} C_{ab}}.
\]
Then
\[\begin{array}{l}
\lim_{\tau\to1^-} \mathcal{RERP}_\tau = \left( \lim_{\tau\to1^-} \frac{\mathcal{R}_\tau^*(P_{\text{PCA}})}{\mathcal{R}_\tau^*(I_d)} - 1 \right) \times 100\%  \\
\ge \left( \frac{\sum_{b \neq m} C_{mb} \pi_b}{\max_{a \neq m, b \neq m} C_{ab}} - 1 \right) \times 100\%.
\end{array}
\]
The cost asymmetry gives $\sum_{b \neq m} C_{mb} \pi_b > \max_{a \neq m, b \neq m} C_{ab}$, so the bound is strictly positive. $\square$
\end{proof}

\paragraph{Example.}
With $C_{mb}=100$, $\max_{a\neq m,b\neq m}C_{ab}=10$, and $\sum_{b\neq m}\pi_b=0.99$, we obtain
\[
\lim_{\tau\to1^-} \mathcal{RERP}_\tau \ge \left( \frac{99}{10} - 1 \right) \times 100\% = 890\%.
\]
Thus PCA is responsible for an 890\% excess risk relative to the optimal baseline.

{\it This leads us to a central takeaway of our analysis, which we capture in the following remark: 99.9999\% Accurate, 890\% Liable: The Tail-Risk Catastrophe of Risk‑Blind Compression. When a PCA classifier achieves near-flawless aggregate scores by compressing data, it frequently blinds itself to the exact tail events that carry the heaviest institutional liability. }

\section{Extension to Risk Shadow Mean-Field-Type Games}

We now examine the multi-agent interaction.  Modern intelligent systems rarely operate on raw observations. Instead, they first compress information through representation-learning mechanisms before making operational decisions.
This observation raises a fundamental question:
\begin{quote}
\textit{Should a machine-intelligence agent optimize its representation for information preservation or for decision responsibility?}
\end{quote}

Classical dimensionality-reduction methods, including PCA, implicitly adopt the first principle by preserving maximal variance. In contrast, responsible MI systems must ultimately be evaluated according to the consequences of their decisions. In legal, regulatory, medical, financial, and safety-critical environments, preserving variance is not the objective. Avoiding harmful decisions is.
To formalize this tension, we introduce a new class of Mean-Field-Type Games (MFTGs, \cite{mftg1,mftg2,audioiamali,audioiamali2,audioiamali3}) in which machine-intelligence agents strategically choose information representations prior to decision making.

\subsection*{Decision Makers}
Consider a finite collection of machine-intelligence agents:
$
\mathcal{I} = \{1, \ldots, I\}, \qquad I \ge 2.
$
Agent $i$ observes a state:
$
x_i \in \mathcal{X}_i \subseteq \mathbb{R}^{d_i},
$
while all agents are influenced by a common environment state:
$
x_0 \in \mathcal{X}_0.
$
The complete state vector is denoted by:
$
x = (x_0, x_1, \ldots, x_I).
$
Unlike classical learning architectures, the agents are not assumed to share identical objectives, identical information structures, or identical decision rules.

\subsection*{Information Structures}
Each agent possesses two possible representation architectures:
$
\mathcal{A}_i = \{\mathrm{PCA}, \mathrm{exp2PCA}\}.
$
A choice $a_i \in \mathcal{A}_i$ determines a projection operator:
$
P_i(a_i, \mu) : \mathcal{X}_i \rightarrow \mathbb{R}^{r_i},
$
where $r_i < d_i$. The compressed information available to agent $i$ is:
$
z_i = P_i(a_i, \mu) x_i.
$
The projection itself may depend on the mean field $\mu$, reflecting the fact that covariance structures, risk structures, and operational environments depend on the collective behavior of the system. The resulting information available to the agent is therefore:
$
\mathcal{F}_i = \sigma(z_i).
$

\subsection*{Actions}
Based on the compressed information $z_i$, the agent selects an action $u_i \in \mathcal{U}_i$. The action may correspond to:
\begin{itemize}
    \item a classification decision;
    \item an intervention or recommendation;
    \item a regulatory action or financial transaction;
    \item an autonomous control command;
    \item a legal or compliance decision.
\end{itemize}

The decision rule is given by:
$
u_i = \gamma_i(z_i).
$
The strategic variable of the agent is therefore:
$
\alpha_i = (a_i, \gamma_i).
$
Unlike classical MFTGs, the strategic variable contains both a \textit{representation mechanism} and a \textit{decision policy}. For this reason, the framework may be viewed as a \emph{Representation-Control Mean-Field-Type Game}.

\subsection*{Mean Field}
The mean field is defined as the joint law:
$
\mu = \mathcal{L} \Big( x_0, x_1, \ldots, x_I, \alpha_1, \ldots, \alpha_I \Big).
$
This formulation follows the finite-player Mean-Field-Type Game paradigm. No asymptotic limit is invoked, no empirical distribution is introduced, no exchangeability assumption is required, and no symmetry assumption is imposed. The mean field captures the collective informational and operational state of the machine-intelligence ecosystem.

\subsection{Payoffs}
The objective of agent $i$ is to minimize:
$
J_i(\alpha, \mu) = J_i \Big( a_1, \ldots, a_I, u_1, \ldots, u_I, \mu \Big).
$
A typical form is:
$
J_i = L_i + \lambda_i R_i + \eta_i K_i + \rho_i G_i,
$
where:
\begin{itemize}
    \item $L_i$ is the expected operational loss;
    \item $R_i$ is a expectile-risk of the cost;
    \item $K_i$ is a representation complexity penalty;
    \item $G_i$ is a governance, legal, ethical, or compliance penalty.
\end{itemize}

The last component is particularly important in responsible machine intelligence. Indeed, many modern MI systems are judged not only by prediction accuracy but also by their legal and societal consequences. Examples include:
\begin{itemize}
    \item false medical diagnoses;
    \item discriminatory lending decisions;
    \item autonomous vehicle accidents;
    \item critical infrastructure failures;
    \item wrongful legal recommendations;
    \item cybersecurity breaches.
\end{itemize}
In such environments, the most damaging events are often rare and contribute negligibly to total variance.

\subsection*{Variance-Oriented and Responsibility-Oriented Agents}
When $a_i = \mathrm{PCA}$, the agent selects a representation according to:
\[
P_i^{\mathrm{PCA}} = \arg\max_{P^\top P = I_{r_i}} \operatorname{Tr}(P^\top \Sigma_i P),
\]
where $\Sigma_i = \operatorname{Cov}(x_i)$. The resulting architecture prioritizes information compression and variance preservation.

Conversely, when $a_i = \mathrm{exp2PCA}$, the representation solves:
\[
P_i^{\mathrm{exp2PCA}} = \arg\max_{P^\top P = I_{r_i}} e_\tau \left( C_i(Y_i, \gamma_i(P^\top x_i)) \right),
\]
where $e_\tau$ denotes the expectile functional. The resulting architecture prioritizes decision responsibility and tail-risk control.
 PCA and exp2PCA represent two algorithms of machine intelligence, as summarized in Table~\ref{tab:competing_philosophies}.

\begin{table}[htbp]
\centering
\caption{Machine Intelligence Representations}
\label{tab:competing_philosophies}
\begin{tabular}{ll}
\toprule
\textbf{PCA Agent} & \textbf{exp2PCA Agent} \\
\midrule
Variance preservation & Risk preservation \\
Compression-oriented & Responsibility-oriented \\
Geometric objective & Operational objective \\
Average behavior & Rare-event behavior \\
Information efficiency & Decision accountability \\
\bottomrule
\end{tabular}
\end{table}

\subsection*{Risk Shadow Equilibrium}
A representation-control equilibrium (mean-field-type equilibrium) is a profile $(\alpha_1^\star, \ldots, \alpha_I^\star)$ such that no agent can improve its payoff through unilateral deviation. Of particular interest are equilibria exhibiting the Risk Shadow phenomenon.

\begin{definition}[Risk Shadow Equilibrium]
A representation-control equilibrium  is called a \emph{Risk Shadow equilibrium} if $(\alpha_1^\star, \ldots, \alpha_I^\star)$ is an mean-field-type equilibrium and  there exists at least one agent $i$ such that:
$
\frac{\operatorname{Var}(P_i^\top X_i)}{\operatorname{Var}(X_i)} \approx 1,
$
while
$
I(Y_i; P_i^\top X_i) \approx 0.
$
\end{definition}

In a Risk Shadow equilibrium, the representation preserves nearly all observable variance while simultaneously destroying the information necessary for responsible decision making.

\subsection*{Implications for Responsible Machine Intelligence}
The proposed framework reveals a fundamental distinction between information preservation and decision responsibility.

Classical machine-intelligence systems implicitly assume that preserving variance preserves value. The present framework demonstrates that this assumption may be false in strategic environments where agents interact through a common mean field and where rare but consequential events dominate societal costs.

The resulting theory therefore suggests a new paradigm:
\begin{quote}
\textit{Machine-intelligence systems should not be evaluated according to how much variance they preserve, but according to how effectively their representations support responsible decisions under risk.}
\end{quote}

This perspective transforms dimensionality reduction from a geometric compression problem into a strategic decision-theoretic problem and establishes a new connection between representation learning, responsible mahine intelligence, and Mean-Field-Type Game theory.
\section{Conclusion and Future Outlook}
\label{sec:conclusion}

This paper began with a simple but consequential question: \emph{Can a dimensionality-reduction method that is optimized solely for variance preservation be trusted in decision systems where the most important events are also the rarest?} The analysis presented throughout this work suggests that, in general, the answer is no. We formally identified and characterized a previously underappreciated phenomenon, which we termed the \emph{Risk Shadow}. The Risk Shadow emerges whenever the geometric objective of variance maximization becomes misaligned with the operational objective of minimizing decision risk. In such settings, PCA may faithfully preserve almost all of the total variance of a dataset while simultaneously discarding the very information that determines the outcome of a high-stakes decision. This is not a numerical accident, nor a pathological corner case. Rather, it is a structural consequence of optimizing an unsupervised geometric criterion in environments where utility, cost, safety, and risk are fundamentally asymmetric. Our theoretical results demonstrate that under severe class imbalance and asymmetric misclassification costs, the directions carrying the largest decision value can reside in subspaces associated with negligible global variance. The classical paradigm of ``maximum variance implies maximum information" breaks down from a decision-theoretic perspective. The retained variance can approach unity while the retained operational value collapses toward zero.

To address this limitation, we investigated three risk-aware alternatives: Tail-Preserving PCA (TP-PCA), Expectile PCA (ExPCA), and the proposed decision-directed framework \textbf{exp2PCA}. The analysis reveals a clear hierarchy. TP-PCA partially compensates for rare-event dilution through sample reweighting, yet remains vulnerable to undesirable centering effects that may redirect the learned subspace toward statistically amplified but operationally irrelevant directions. ExPCA provides a substantially more principled remedy by embedding tail asymmetry directly into the geometric representation process, thereby recovering critical structures that classical PCA systematically suppresses. Exp2PCA moves beyond geometric surrogates altogether. By optimizing the downstream decision objective itself, it establishes a direct connection between representation learning and operational expectile risk, achieving universal task alignment and eliminating the fundamental source of the Risk Shadow. The empirical evidence reinforces these theoretical conclusions. Across high-dimensional synthetic experiments and a diverse collection of real-world domains including fraud detection, industrial fault diagnosis, cybersecurity monitoring, and medical imaging, expectile risk-aware representations consistently recovered decision-critical information that variance-based approaches either attenuated or completely removed. The improvements were particularly pronounced in precisely those regimes where mistakes are the most costly: rare, consequential, and asymmetric events.

At the same time, intellectual honesty requires acknowledging what this work does \emph{not} prove. We do not claim that PCA is obsolete, nor that variance maximization is inherently flawed. PCA remains one of the most elegant, interpretable, and computationally efficient tools in modern {\it data intelligence}. When the objective is compression, visualization, denoising, or exploratory analysis, PCA often remains an excellent choice. Our claim is narrower but more important: \emph{when dimensionality reduction serves as a precursor to consequential decision making, variance alone is generally an insufficient design principle}. In such contexts, preserving what is common is not necessarily the same as preserving what matters.

The results suggest a conceptual shift for representation learning. Future dimensionality-reduction systems should not be evaluated solely by how much variance they retain, but by how much decision-relevant information they preserve. The central question is no longer ``What is the dominant geometry of the data?"  but rather ``What information is indispensable for the decisions that follow?"

\subsection*{Limitations and Future Research Directions}

While the theoretical framework developed here establishes a  foundation for  expectile risk-aware dimensionality reduction, several important challenges remain open. The effectiveness of tail-sensitive methods depends on the asymmetry parameter $\tau$. Although large values of $\tau$ enhance sensitivity to rare events, excessively extreme choices may increase estimator variability and amplify sampling noise in finite datasets. Developing statistically principled, data-adaptive procedures for selecting $\tau$ remains a central challenge. Future work should investigate bootstrap-based calibration, information criteria,  distributionally robust selection mechanisms, and tau-inverse expectile.
The present theory is developed primarily within linear projection frameworks, where interpretability and analytical tractability are strongest. Many modern datasets, however, reside on highly nonlinear manifolds. Extending expectile-based and decision-directed objectives to deep representation-learning architectures, including autoencoders, diffusion embeddings, and latent-variable models, represents a natural and potentially transformative next step.

Our  results establish structural properties of the Risk Shadow, but a complete statistical theory remains incomplete. Sharp finite-sample guarantees, concentration inequalities, excess-risk bounds, and uniform convergence results over constrained projection manifolds are still largely unexplored. Addressing these questions will require new tools that combine empirical process theory with asymmetric and tail-sensitive objectives. Many critical applications operate in continuously evolving environments where observations arrive sequentially and rare events emerge unexpectedly. Developing recursive, online, and adaptive variants of ExPCA and exp2PCA capable of dynamically reallocating representational capacity toward emerging risks is an important direction for both theory and deployment.
Perhaps the most important open question is broader than dimensionality reduction itself. The Risk Shadow identified in this work may represent only one manifestation of a more general phenomenon: the systematic disappearance of rare but consequential signals under objectives dominated by averages. Investigating analogous effects in foundation models, large-scale multimodal systems, reinforcement learning, and autonomous decision platforms may reveal a wider class of hidden vulnerabilities that current optimization paradigms fail to detect.

This work argues that the dominant paradigm of unsupervised variance preservation is insufficient whenever the cost of overlooking a rare event exceeds the benefit of accurately representing the common case. The Risk Shadow is not merely a limitation of PCA; it is a warning about the broader consequences of optimizing for what is frequent rather than what is consequential. As machine intelligence systems increasingly influence critical societal decisions, the ability to preserve low-probability, high-impact information may become not merely a statistical preference, but a fundamental requirement for trustworthy machine intelligence.

\subsection*{Acknowledgments}
We thank the anonymous reviewers for their insightful comments. We would also like to express our sincere gratitude to the data annotators, data collectors, and data enablers from {\it TIMADIE: Guinaga, Grabal, SK1 Sogoloton, LnG Lab, MFTG, CI4SI, WETE, IA Mali, TimaTon, ORA, and TommoSo Dogon} for their invaluable contributions to this work.

\subsubsection*{Data Ethics}
The collection, curation, and deployment of the datasets referenced in this work were conducted in strict compliance with established ethical guidelines and institutional protocols. All data obtained via our collaborative networks including {\it TIMADIE: Guinaga, Grabal, SK1 Sogoloton, LnG Lab, MFTG, CI4SI, WETE, IA Mali, TimaTon, ORA, and TommoSo Dogon}, adhered to strict frameworks of informed consent, privacy preservation, and data minimization. No personally identifiable information (PII) or sensitive demographic attributes were processed without explicit authorization, and  anonymization protocols were applied prior to analysis. The authors ensure that the experimental methodologies do not perpetuate, introduce, or amplify historical biases, and explicitly restrict the use of these models from applications involving discriminatory profiling or harmful automated decision-making.

\subsubsection*{Data Availability}
All datasets analyzed or generated during the course of this study are publicly available online. Complete source repositories, configuration files, and preprocessing scripts are fully documented and cited within the text, with corresponding persistent identifiers provided explicitly in the references section.

\subsubsection*{Use of Generative AI}
During the preparation of this manuscript, the authors utilized advanced generative artificial intelligence platforms and large language models solely for the purposes of code optimization, grammatical refinement, and structural formatting of the text. All technical ideas, mathematical proofs, experimental formulations, and interpretations of results remain entirely the original work of the human authors, who accept full responsibility and accountability for the accuracy, integrity, and authenticity of the final content.

\subsubsection*{Referential Hallucination}
To ensure the absolute integrity and reproducibility of this work, every bibliographic reference cited in this manuscript was manually verified one by one by the authors. The authors cross-checked all metadata including authors, titles, journals, publication years, volumes, issue numbers, and persistent digital object identifiers (DOIs) against primary academic databases. This  manual audit confirms that all cited literature exists, is peer-reviewed, and directly supports the technical claims, mathematical proofs, and contextual frameworks presented herein, thereby guaranteeing complete protection against artificial intelligence-generated referential hallucinations.

\subsubsection*{Conflict of Interest}
The authors declare that they have no known competing financial interests, personal relationships, or institutional affiliations that could have appeared to influence the work, methodology, or conclusions reported in this paper.

\bibliographystyle{plain}

\newpage 
\section*{Biography}

 \begin{IEEEbiography}[{\includegraphics[width=1in,clip,keepaspectratio]{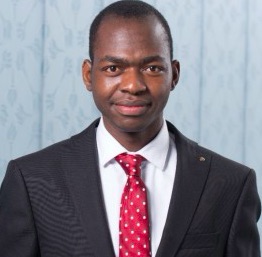}}]{Hamidou Tembine} (SM'13) is the co-founder of Timadie, Grabal  AI Mali,  co-chair of TF, founder of Guinaga, WETE, MFTG, LnG Lab, CI4SI and a Professor of Artificial Intelligence at UQTR, Quebec, Canada. He graduated in Applied Mathematics from Ecole Polytechnique (Palaiseau, France) and received the Ph.D. degree from INRIA and University of Avignon, France. He further received his Master degree in game theory and economics. His main research interests are learning, evolution, and games. In 2014, Tembine received the IEEE ComSoc Outstanding Young Researcher Award for his promising research activities for the benefit of society. He was the recipient of 10+ best paper awards in the applications of game theory. Tembine is a prolific researcher and holds 300+ scientific publications including magazines, letters, journals and conferences. He is author of the book on ``distributed strategic learning for engineers" (published at CRC Press, Taylor \& Francis 2012) which received book award 2014, and co-author of the book ``Game Theory and Learning in Wireless Networks'' (Elsevier Academic Press) and co-author of the book on ``Mean-Field-Type Games I-II " (Springer Nature)  and  for Engineers' (CRC Press0, and the author of the book ``GPT Meets Game Theory".  Tembine has been co-organizer of several scientific meetings on game theory in agriculture, water, food, environment, networking, wireless communications and smart energy systems. He has been a visiting researcher at University of California at Berkeley (US), University of McGill (Montreal, Quebec, Canada), University of Illinois at Urbana-Champaign (UIUC, US), Ecole Polytechnique Federale de Lausanne (EPFL, Switzerland) and University of Wisconsin (Madison, US). He has been a Simons Participant and a Senior Fellow 2020. He is a senior member of IEEE. He is a Next Einstein Fellow, Class of 2017.
 \end{IEEEbiography}

\end{document}